\documentclass[11pt]{article}
\usepackage{amsmath}
\usepackage{amsfonts}
\usepackage{amssymb}
\usepackage{mathtools}
\usepackage{amsthm}
\usepackage{multidef}
\usepackage{dsfont}
\usepackage{booktabs}
\usepackage{bm}
\usepackage{mleftright}
\usepackage{makecell}

\usepackage{array}
\usepackage[caption=false,font=normalsize,labelfont=sf,textfont=sf]{subfig}
\usepackage{textcomp}
\usepackage{stfloats}
\usepackage{url}
\usepackage{verbatim}
\usepackage{graphicx}
\usepackage{cite}
\usepackage[letterpaper,margin=0.5in]{geometry}

\usepackage{algorithm}
\usepackage[noend]{algpseudocode}
\algrenewcommand\algorithmicrequire{\textbf{Input}:}
\algrenewcommand\algorithmicensure{\textbf{Output}:}
\algnewcommand{\IfThen}[2]{
  \State \algorithmicif\ #1\ \algorithmicthen\ #2}

\usepackage[hidelinks]{hyperref}
\usepackage{enumitem}

\algdef{SE}[DOWHILE]{Do}{doWhile}{\algorithmicdo}[1]{\algorithmicwhile\ #1}%

\usepackage{colortbl}%

\DeclareMathOperator*{\Ran}{\mathrm{Ran}}
\DeclareMathOperator*{\Rank}{\mathrm{Rank}}

\usepackage{tikz}
\usepackage{pgfplots}
\usepackage[Tol,OkabeIto,pgf]{colorblind}
\usetikzlibrary{arrows.meta, positioning, calc, fit}
\usetikzlibrary{plotmarks}
\usetikzlibrary{spy}
\usepgfplotslibrary{groupplots}
\usepgfplotslibrary{colorbrewer}
\pgfplotsset{
    compat=newest,
    MyLinePlot/.style={
        line width=0.6pt,
        mark size=1.2pt,
        mark options={solid, line width=0.75pt}
  }
}

\newtheorem{theorem}{Theorem}

\newtheorem{proposition}[theorem]{Proposition}
\newtheorem{corollary}[theorem]{Corollary}

\theoremstyle{definition}
\newtheorem{problem}{Problem}

\theoremstyle{remark}
\newtheorem{remark}{Remark}

\multidef[prefix=set]{\ensuremath{\mathbb{#1}}}{A-Z}

\multidef[prefix=cal]{\ensuremath{\mathcal{#1}}}{A-Z}

\multidef[prefix=frak]{\ensuremath{\mathfrak{#1}}}{A-Z}

\multidef[prefix=v]{\ensuremath{\mathbf{#1}}}{a-z}

\multidef[prefix=m]{\ensuremath{\mathbf{#1}}}{A-Z}
\newcommand{\mDelta}{\mathbf{\Delta}}
\newcommand{\mLambda}{\mathbf{\Lambda}}

\newcommand{\proposed}{\textrm{G-REST}}
\date{}

\begin{document}

\title{Subspace Projection Methods for Fast Spectral \\Embeddings of Evolving Graphs}

\author{Mohammad Eini, Abdullah Karaaslanli, Vassilis~Kalantzis
        and~Panagiotis~A.~Traganitis
\thanks{M. Eini, A. Karaaslanli, and P.~A.~Traganitis are with the Department of Electrical and Computer Engineering, Michigan State University, East Lansing, MI, 48824, USA. E-mails: \{einimoha,karaasl1,traganit\}@msu.edu.
V.~Kalantzis is with IBM Research. E-mail: vkal@ibm.com.
The work of M. Eini, A. Karaaslanli, and P.~A.~Traganitis is supported by NSF award 2312546.
This work was supported in part through computational resources and services provided by the Institute for Cyber-Enabled Research at Michigan State University.}}
\markboth{Journal of \LaTeX\ Class Files,~Vol.~14, No.~8, August~2021}%
{Shell \MakeLowercase{\textit{et al.}}: A Sample Article Using IEEEtran.cls for IEEE Journals}


\maketitle

\begin{abstract}
Several graph data mining, signal processing, and machine learning downstream tasks rely on information related to the eigenvectors of the associated adjacency or Laplacian matrix. Classical eigendecomposition methods are powerful when the matrix remains static but cannot be applied to 
problems where the matrix entries are updated or the number of rows and columns increases frequently. Such 
scenarios occur routinely in graph analytics when the graph is changing dynamically and either 
edges and/or nodes are being added and removed. This paper puts forth a new algorithmic framework 
to update the eigenvectors associated with the leading eigenvalues of an initial adjacency or Laplacian matrix as the  graph evolves dynamically. The proposed algorithm is based on Rayleigh-Ritz projections, in which the original eigenvalue problem is projected onto a restricted subspace which 
ideally encapsulates the invariant subspace associated with the sought eigenvectors. Following ideas from eigenvector perturbation analysis, we present a new methodology to build the projection subspace. 
The proposed framework features lower computational and memory complexity with respect to competitive alternatives while empirical results show strong qualitative performance, both in terms of eigenvector approximation and accuracy of downstream learning tasks of central node identification and node clustering.
\end{abstract}

\section{Introduction}

Graphs are ubiquitous in a range of signal processing and machine learning tasks as they provide a natural framework for representing data defined on irregular, non-Euclidean domains by encoding 
pairwise relationships and interactions among entities. Graphs may arise naturally, such as when 
modeling communication, social, power, and sensor networks, or may be artificially constructed to 
capture similarities or relationships between components of systems or data. To enable processing and inference over graph data, a multitude of algorithms and methodologies have emerged from the signal processing, data mining, and machine learning communities, spanning approaches from statistical techniques to modern deep learning frameworks~\cite{geometric_DL2017,gsp_spm_2023}. A plethora of such methods rely on so-called node embeddings, which are vector representations of the entities in a graph. Such embeddings encode topological information of the graph and enable signal processing and machine learning algorithms, designed for vector data, to be applied to graph data, facilitating tasks such as node clustering, ranking, visualization, and prediction~\cite{network_representation_survey_2020,Chen_Wang_Wang_Kuo_2020,cai2018comprehensive,xu2021understanding,kollias2022directed}.

\emph{Spectral embeddings} or \emph{eigenembeddings}, i.e., embeddings that rely on the eigendecomposition of a graph operator, play a prominent role. Eigenvalues and eigenvectors of the adjacency matrix of a graph are associated with fundamental topological graph properties, such as connectivity, edge density, conductance~\cite{farkas_spectra_2001}. Adjacency-based spectral embeddings underpin applications such as \emph{node or vertex centrality}~\cite{benzi2013ranking,kang2011centralities}, including the notable PageRank algorithm~\cite{page1999pagerank}, and subgraph  centralities~\cite{estrada2005subgraph}, \emph{clustering} \cite{spectralclustering,vatrapu2016social}  and \emph{visualization and positioning in dot-product graphs}~\cite{jmlr_randomdotproduct_2018,gd_based_spectral_embeddings_TSP24,marenco2025weightedrandomdotproduct} among others. 
Accordingly, eigenvalues and eigenvectors of the Laplacian matrix of a graph are related to a multitude of applications such as community detection and clustering of graphs~\cite{spectralclustering}, random walks on graphs~\cite{lovasz1993random}, and graph filtering~\cite{gsp_spm_2023,isufi_tsp_graphfilters_2024} to name a few. As these eigenembeddings rely on the eigendecomposition of the adjacency or the Laplacian of a graph, they can be easily and efficiently computed using contemporary numerical linear algebra tools, especially when the graph operator is sparse and only a few eigenvectors are required.

Nevertheless, oftentimes the graph of interest may change or evolve over time~\cite{evolutionary_network_analysis}. For instance, in social networks new connections may be formed while existing ones may vanish, and new users may join the network, introducing additional nodes and links \cite{community_dynamic_networks2018}. Communication and sensor networks undergo topological changes, as sensors or communication nodes may disconnect or connect from the network, due to mobility, energy constraints or failures. In neuroscience, the temporal structure of brain connectivity networks can indicate dynamic interactions between different brain regions across different tasks~\cite{Bullmore2009-ag}. 

In such temporal or evolving graph scenarios, it is desirable to track key graph properties or perform learning via the embeddings of the changing graph~\cite{dynamic_network_embedding_survey_2022}. However, recomputing embeddings from scratch whenever the graph changes can incur prohibitive computational complexity, especially if the changes occur frequently. 
Thus, methods that efficiently track the embeddings of evolving graphs, foregoing the need for full computation at every time step, are well motivated. Focusing on the eigenvector embeddings of a graph, this work puts forth a novel algorithm, based on the Rayleigh-Ritz approximation scheme, that can readily track the leading eigenpairs of the adjacency matrix of an evolving graph. Specifically, the contributions of this work are as follows:
\begin{itemize}
\item Characterization of perturbation-based eigenpair tracking algorithms from a subspace- and projection-based perspective;
\item Development of a highly accurate and efficient method to track the dominant eigenvalues and eigenvectors of evolving graphs, capable of handling node additions and changes in graph connectivity;
\item Extension of the proposed algorithm to the tracking of eigenpairs of Laplacian matrices;
\item Extensive numerical tests with synthetic and real graphs, evaluating eigenvector estimation quality and performance on two important downstream tasks.
\end{itemize}

The remainder of the paper is organized as follows: Sec.~\ref{sec:problem_statement_prelims} outlines the problem statement and prior art; Sec.~\ref{sec:projections_main} describes the proposed method for updating the eigenembeddings of the graph adjacency for a single update, while Sec.~\ref{sec:main_graph_updates} extends the results of Sec.~\ref{sec:projections_main} to multiple graph updates and to the case of Laplacian matrices. Section~\ref{sec:numerical-tests} evaluates the performance of the proposed algorithm via extensive numerical tests on real and synthetic graphs. Finally, concluding remarks and future research directions are given in Section \ref{sec:conclusion}.

\subsection{Notation} 
Unless otherwise noted, lowercase bold letters, $\mathbf{x}$, denote vectors, uppercase bold letters, $\mathbf{X}$, represent matrices, and calligraphic uppercase letters, $\mathcal{X}$, stand for sets. The $(i,j)$th entry, $i$th row and $i$th column of matrix $\mathbf{X}$ are denoted by $[\mathbf{X}]_{ij}$, $[\mX]_{i\cdot}$ and $[\mX]_{\cdot i}$, respectively. $\Ran(\mathbf{X})$ and $\Rank(\mathbf{X})$ denote the range-space and rank of $\mathbf{X}$, respectively, and $|\mathcal{X}|$ denotes the cardinality of set ${\cal X}$. The difference between two sets $\cal A$ and $\cal B$ is denoted as $\cal A \setminus B$.  ‘$\oplus$' denotes the direct sum between two subspaces, and $\mathbf{I}$ and $\mathbf{0}$ the identity and all-zero matrices of appropriate dimension, respectively. When the dimension of $\mI$ and $\mathbf{0}$ is not clear from context, they are respectively typed as $\mI_N$ and $\mathbf{0}_{N,M}$ to indicate their dimensions. 
Finally, ${\rm nnz}(\mathbf{X})$ denotes the number of non-zero entries of the matrix $\mathbf{X}.$

\section{Problem statement and preliminaries}
\label{sec:problem_statement_prelims}

An undirected graph $\calG$ is a mathematical object consisting of the tuple $\calG = (\calV, \calE)$ where $\calV$ denotes the set of $|\calV| = N$ nodes and  $\calE \coloneqq \{(i,j) | (i,j)\in{\cal V}\times {\cal V}, i \neq j\}$ is the set of edges whose elements represent the links between nodes of $\calG$. The connectivity of $\calG$ can be encoded in a symmetric adjacency matrix $\mA \in \setR^{N \times N}$ where $[\mA]_{ij} = [\mA]_{ji} = 1$ if $(i,j) \in \calE$ and $[\mA]_{ij}=[\mA]_{ji}=0$, otherwise. The $i$th eigenvalue and corresponding eigenvector of $\mA$ are denoted as $\lambda_i$ and $\vx_i$, respectively, and the adjacency matrix admits the eigendecomposition $\mA = \mX \mLambda \mX^\top$, where $\mX = [\vx_1,\vx_2,\ldots,\vx_N]$ collects the eigenvectors of $\mA$ as columns, and $\mLambda$ is a diagonal matrix whose entries are the corresponding eigenvalues, i.e., $[\mLambda]_{ii} = \lambda_i$. The eigenvalue $\lambda_i$ and its corresponding eigenvector $\vx_i$ define the $i$th \textit{eigenpair} $(\lambda_i,\mathbf{x}_i)$ of $\mA$. Unless mentioned otherwise, the eigenpairs of any matrix are ordered according to the magnitude of their respective eigenvalues, i.e., $|\lambda_1| \geq |\lambda_2| \geq \dots |\lambda_N|$ in the case of $\mA$.

\begin{figure}[t]
    \centering
    \begin{tikzpicture}[
    font=\footnotesize,
    >=Latex,
    nodeOld/.style={circle, draw, fill=gray!20, inner sep=2pt},
    nodeNew/.style={circle, draw, fill=orange!35, inner sep=2pt},
    edgeKeep/.style={draw=gray!70, line width=0.8pt},
    edgeAdd/.style={draw=green!60!black, line width=1.2pt},
    edgeDel/.style={draw=red!70!black, dashed, line width=1.0pt}
]

\begin{scope}[shift={(0,0)}, scale=0.8]
  \node at (-0.1,2.4) {$\mathcal{G}^{t}$};
  \draw (-1.9,2) -- (1.7,2);

  \node[nodeOld] (a) at (-1.7,0.8) {$1$};
  \node[nodeOld] (b) at (-0.7,1.5) {$2$};
  \node[nodeOld] (c) at (0.6,1.2)  {$3$};
  \node[nodeOld] (d) at (1.5,0.2)  {$4$};
  \node[nodeOld] (e) at (0.4,-0.9) {$5$};
  \node[nodeOld] (f) at (-1.1,-0.8) {$6$};

  \draw[edgeKeep] (a) -- (b);
  \draw[edgeKeep] (b) -- (c);
  \draw[edgeKeep] (c) -- (d);
  \draw[edgeKeep] (d) -- (e);
  \draw[edgeKeep] (e) -- (f);
  \draw[edgeKeep] (f) -- (a);
  \draw[edgeKeep] (b) -- (e);
\end{scope}

\draw[->, line width=1.0pt] (1.6,0.3) -- (3.3,0.3)
  node[midway, above, align=center] {Graph\\Update};

\begin{scope}[shift={(4.8,0)}, scale=0.8]
  \node at (0.35,2.4) {$\mathcal{G}^{t+1}$};
  \draw (-1.9,2) -- (2.6,2);

  \node[nodeOld] (a2) at (-1.7,0.8) {$1$};
  \node[nodeOld] (b2) at (-0.7,1.5) {$2$};
  \node[nodeOld] (c2) at (0.6,1.2)  {$3$};
  \node[nodeOld] (d2) at (1.5,0.2)  {$4$};
  \node[nodeOld] (e2) at (0.4,-0.9) {$5$};
  \node[nodeOld] (f2) at (-1.1,-0.8) {$6$};

  \node[nodeNew] (g2) at (2.2,1.3) {$7$};
  \node[nodeNew] (h2) at (2.4,-0.6) {$8$};

  \draw[edgeKeep] (a2) -- (b2);
  \draw[edgeKeep] (b2) -- (c2);
  \draw[edgeKeep] (c2) -- (d2);
  \draw[edgeKeep] (d2) -- (e2);
  \draw[edgeKeep] (e2) -- (f2);
  \draw[edgeKeep] (f2) -- (a2);

  \draw[edgeDel] (b2) -- (e2);

  \draw[edgeAdd] (a2) -- (c2);
  \draw[edgeAdd] (c2) -- (e2);

  \draw[edgeAdd] (g2) -- (c2);
  \draw[edgeAdd] (g2) -- (d2);
  \draw[edgeAdd] (h2) -- (e2);
  \draw[edgeAdd] (h2) -- (d2);
\end{scope}

\begin{scope}[shift={(-1.3,-3)}, scale=0.8]
    \draw[draw=black!20, rounded corners=2pt, fill=white] (0,0.4) rectangle ++(10,1.6);
    \node[nodeOld, minimum size=8pt] (l1) at (0.5,1.5) {};
    \node[right=2pt of l1] {Existing Node};
    \node[nodeNew, minimum size=8pt] (l2) [right=2.5 of l1] {};
    \node[right=2pt of l2] {New Node};
    \draw[edgeKeep] ($(l1.west) + (0, -0.75)$) -- ($(l1.east) + (0, -0.75)$);
    \node[anchor=west] at ($(l1.east) + (2pt, -0.75)$) {Unchanged Edge};
    \draw[edgeAdd] ($(l2.west) + (0, -0.75)$) -- ($(l2.east) + (0, -0.75)$);
    \node[anchor=west] at ($(l2.east) + (2pt, -0.75)$) {Added Edge};
    \draw[edgeDel] ($2*(l2.west) - (l1.west) + (-0.5, -0.75)$) -- ($2*(l2.east) - (l1.east) + (-0.25, -0.75)$);
    \node[anchor=west] at ($(l2.east) + (l2.west) - (l1.west) + (-0.25, 0) +(2pt, -0.75)$) {Deleted Edge};
\end{scope}

\end{tikzpicture}
    \caption{Graph evolution from timestep $t$ to $t+1$. The updated graph $\mathcal{G}^{(t+1)}$ features edge deletions (dashed), edge additions (bold), and newly introduced vertices. In this example: $[\mathbf{K}^{(t+1)}]_{1,3} = [\mathbf{K}^{(t+1)}]_{3,1} = 1$, $[\mathbf{K}^{(t+1)}]_{3,5} = [\mathbf{K}^{(t+1)}]_{5,3} = 1$, $[\mathbf{K}^{(t+1)}]_{2,5} = [\mathbf{K}^{(t+1)}]_{5,2} = -1$, $[\mathbf{G}^{(t+1)}]_{3,1} = 1$, $[\mathbf{G}^{(t+1)}]_{4,1} = 1$, $[\mathbf{G}^{(t+1)}]_{4,2} = 1$, and $[\mathbf{G}^{(t+1)}]_{5,2} = 1$.}
    \label{fig:graph-evolution}
\end{figure}
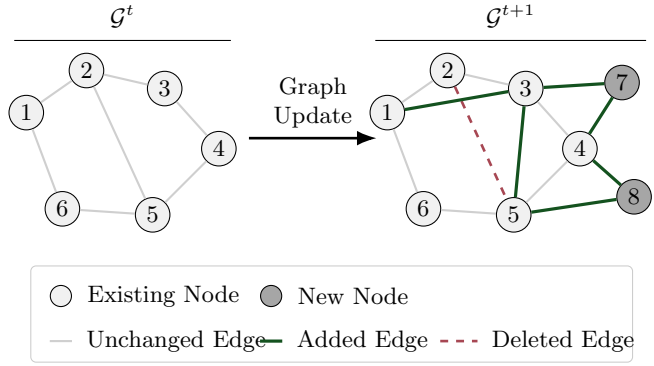

In dynamic or temporal graph settings, a graph can change over time and is observed at discrete-time steps. This results in a sequence of graphs $\{{\cal G}^{(t)}\}_{t=0}^{T}$, with $\calG^{(t)} = (\calV^{(t)}, \calE^{(t)})$ denoting the graph observed at time-step $t$. Similarly to the non-dynamic (static) case, the cardinality of the vertex set ${\cal V}^{(t)}$ is denoted as $|\calV^{(t)}| = N^{(t)}$ while the connectivity of ${\cal G}^{(t)}$ is encoded in an $N^{(t)}\times N^{(t)}$ adjacency matrix $\mathbf{A}^{(t)}$. The $j$th \emph{eigenpair} of $\mathbf{A}^{(t)}$ is denoted as $(\lambda_{j}^{(t)},\mathbf{x}_{j}^{(t)})$. The transition from $\calG^{(t)}$ to $\calG^{(t+1)}$ occurs via either (or a combination) of the following mechanisms:
\begin{itemize}
    \item \textbf{Topological updates:} Additions of new edges and/or removal of existing edges among nodes in $\calV^{(t)}$.
    \item \textbf{Graph expansion:} Introduction of a new set of nodes $\calC^{(t+1)}$ with $|\calC^{(t+1)}| = S^{(t+1)}$, resulting in an expanded vertex set $\calV^{(t+1)} = \calV^{(t)} \cup \calC^{(t+1)}$.
\end{itemize}
Based on this transition model, the $(N^{(t)}+S^{(t+1)}) \times (N^{(t)} + S^{(t+1)})$ adjacency matrix 
$\mA^{(t+1)}$ can be expressed as:
\begin{align}
\label{eq:adj_update}
& \mathbf{A}^{(t+1)} = 
    \overline{\mA}^{(t)} + \mDelta^{(t+1)} \\
&=
    \begin{bmatrix} 
        \mathbf{A}^{(t)} & \mathbf{0}_{N^{(t)},S^{(t+1)}} \\ 
        \mathbf{0}_{S^{(t+1)},N^{(t)}} & \mathbf{0}_{S^{(t+1)},S^{(t+1)}}
    \end{bmatrix} +
    \begin{bmatrix}
        \mathbf{K}^{(t+1)} & \mathbf{G}^{(t+1)} \\ 
        {\mathbf{G}^{(t+1)\top}} & \mathbf{C}^{(t+1)}
    \end{bmatrix}, \notag
\end{align}
where $\overline{\mA}^{(t)}$ denotes the padding of $\mathbf{A}^{(t)}$ by $S^{(t+1)}$ 
identically zero rows and columns and $\mDelta^{(t+1)}$ encodes the connectivity changes that have occurred from $\calG^{(t)}$ to $\calG^{(t+1)}$. In particular, $\mathbf{K}^{(t+1)} \in \{-1,0,1\}^{N^{(t)}\times N^{(t)}}$ is a sparse matrix representing topological updates, i.e., $[\mathbf{K}^{(t+1)}]_{ij} = 1$ $\left([\mathbf{K}^{(t+1)}]_{ij} = -1\right)$ if and only if the previously unconnected (connected) vertices $i$ and $j$ of $\calV^{(t)}$ become connected (disconnected). Similarly, $\mathbf{G}^{(t)} \in \{0,1\}^{N^{(t)}\times S^{(t+1)}}$ encodes the connectivity between newly added nodes and existing nodes, i.e.,   $[\mathbf{G}^{(t+1)}]_{ij}=1$ if and only if the $i$th vertex of ${\calV}^{(t)}$ is connected to the $j$th node of $\calC^{(t+1)}$. Finally, $\mC^{(t+1)} \in \{0,1\}^{S^{(t+1)}\times S^{(t+1)}}$ encodes edges among newly added nodes, i.e., $[\mathbf{C}^{(t+1)}]_{ij}=1$ if and only if the $i$th and $j$th vertices of $\calC^{(t+1)}$ are connected. An example of a graph update $\calG^{(t)}\rightarrow \calG^{(t+1)}$ and corresponding changes in its adjacency matrix is illustrated in
Fig.~\ref{fig:graph-evolution}.

\subsection{Problem statement}
\label{ssec:problem_statement}

The goal of this paper is to \textit{efficiently} update the leading $K$ eigenvalues and 
corresponding eigenvectors of time-evolving graph adjacency matrices. Specifically, 
for $j=1, \dots, K$, the goal is to update the eigenpair \smash{$(\lambda_{j}^{(t)},\mathbf{x}_{j}^{(t)})$} 
of $\mathbf{A}^{(t)}$ to the corresponding unknown eigenpair $(\lambda_{j}^{(t+1)},\mathbf{x}_{j}^{(t+1)})$ 
of $\mathbf{A}^{(t+1)}$, for $t=0, \ldots, T-1$. Generally, the transformation of nodes of a graph 
into low-dimensional numerical vectors is known as \emph{node embedding} \cite{xu2021understanding}. 
Herein, our focus lies 
exclusively on algebraic problems, i.e., no specific geometric/connectivity characteristics 
in either the graph updates or the spectral properties of the adjacency matrices are assumed.

Since graph updates $\calG^{(t)} \rightarrow \calG^{(t+1)}$ for $\ t=0, \ldots, T-1$, occur sequentially, 
the main goal is to develop an efficient  embedding update algorithm for a generic adjacency update problem. This update algorithm can then be applied $T$ times. For simplicity, the remainder of this section and Sec.~\ref{sec:projections_main} focus on a single time-step update. Let $\mathbf{A}$ denote a $N\times N$ 
sparse and symmetric adjacency matrix and consider its update:
\begin{equation}
\label{eq:adj-update-single}
\widehat{\mA} = \overline{\mA} + \mDelta = 
    \begin{bmatrix} 
        \mathbf{A} & \mathbf{0}_{N,S} \\ 
        \mathbf{0}_{S,N} & \mathbf{0}_{S,S}
    \end{bmatrix} +
    \begin{bmatrix}
        \mathbf{K} & \mathbf{G} \\ 
        \mathbf{G}^\top & \mathbf{C}
    \end{bmatrix}.
\end{equation}
Here, $\overline{\mA}$ denotes the extension of $\mA$ by $S$ identically zero rows and columns and $\mDelta$ encodes the incoming update. The update matrix 
$\mDelta$ corresponds to the addition/deletion of graph edges after the extension of the graph with $S$ nodes. With $(\lambda_j, \vx_j)$ denoting the $j$th eigenpair of $\mA$, The following eigenvalue equation is trivially satisfied for $\overline{\mA}$:
\begin{equation}
    \overline{\mA} \overline{\vx}_j = \lambda_j \overline{\vx}_j,
\end{equation}
where $\overline{\mathbf{x}}_j=\begin{bmatrix} \mathbf{x}_j\\
\mathbf{0}_{S,1}\end{bmatrix}$. 
The problem of each time-step update $t\rightarrow t+1$ of the $K$ leading eigenpairs can then be stated as:
\begin{problem}
\label{prob:st-update}\emph{
Given the eigenpairs $\{(\lambda_{j}, \vx_{j})\}_{j=1}^K$ of the $N\times N$ matrix ${\mA}$ and a $(N+S)\times(N+S)$ perturbation 
matrix $\mDelta$, estimate the eigenpairs 
$\{(\widehat{\lambda}_{j}, \widehat{\mathbf{x}}_{j})\}_{j=1}^{K}$ of the matrix $\widehat{\mathbf{A}}$.}
\end{problem}

Note that, while the ensuing discussion focuses on binary adjacency matrices, the methods developed can be readily applied to graphs associated with weighted adjacency matrices.

\subsection{Prior art}
\label{ssec:prior_art}
The most straightforward approach to compute the $K$ leading eigenpairs of the  
matrix $\widehat{\mA}$ is to compute them from scratch, via direct eigendecomposition, ignoring the pre-computed eigenvectors of $\mathbf{A}$. A Lanczos-based solver ~\cite{simon1984lanczos} incurs a computational complexity of  $O(|{\cal E}^{(t+1)}|K+(N^{(t)}+S^{(t+1)})K^2)$, rendering this approach  inefficient 
for frequent graph updates \cite{kalantzis2023matrix}.
Algorithms that approximate the $K$ leading eigenpairs of  $\widehat{\mA}$ 
via updating those of $\overline{\mA}$ are generally trading accuracy for speed. One 
such class of algorithms is based on \emph{eigendecomposition tracking}, by adapting the SVD-updating schemes of \cite{zha1999updating,kalantzis2021projection,vecharynski2014fast}. 
Specifically, \cite{zha1999updating} describes three separate methods to update 
$K$-leading singular values and singular vectors of a matrix  for column-wise 
augmentation, row-wise augmentation, and low-rank updates. Similar ideas to update 
the $K$-smallest eigenpairs of normalized Laplacian matrices are considered in \cite{dhanjal2014efficient} for three types of dynamic graph evolution: topological updates, node addition, and node removal. In \cite{deng2024efficient}, a modification of \cite{dhanjal2014efficient} is developed for the case of node addition by introducing 
an error-correction step. Other algorithms to update eigenvector and singular vector subspaces, 
subject to matrix perturbations and augmentations, can be found in \cite{moonen1992singular,hall2002adding,brand2002incremental,brand2003fast,degroat2002efficient}. Similarly, classical algorithms for the problem of subspace tracking and signal estimation for sliding and sliding windows can be found in \cite{badeau2005fast,abed2002fast,chi2013petrels} while algorithms for online eigenspace monitoring and detection of change in sequences of random 
graphs can be found in \cite{marenco2022online,fiori2023gradient}. Nevertheless, subspace tracking algorithms typically do not handle cases where the dimensionality of the subspace increases, corresponding to node additions.

A second class of algorithms to update the leading eigenpairs is based on 
constructing quadratic-error approximations of $\widehat{\lambda}_j$ and 
$\mathbf{\widehat{x}}_j$ by multiplying the perturbation matrix 
$\mathbf{\Delta}$ with the eigenvectors of $\mathbf{\overline{A}}$.
Instead of re-computing the entire eigendecomposition, algorithms based on first-order sensitivity analysis approximate the updated eigenvectors by applying corrections derived from the update matrix itself. As a result, such approaches can be appealing 
when the adjacency (or Laplacian) matrix is updated sparsely. Research on eigenpair updates for dynamic graphs focuses on maintaining spectral information efficiently as graph topologies evolve.
The work in \cite{chen2015fast,chen2017eigenfunctions} utilizes  matrix perturbation theory to update top-$K$ eigenpairs with linear complexity, specifically targeting edge-level modifications. Building on the concept of approximation, the work in \cite{mitz2022perturbation} advanced the use of sensitivity-based frameworks within kernel approximations, demonstrating how spectral properties can be maintained even under significant data perturbations, by focusing on the stability of the underlying operator. However, as will be shown in the ensuing sections, these first-order approaches ignore information from new nodes that are added to a graph. A first effort to extend perturbation-based algorithms to projection subspaces via resolvent-based approaches was presented in \cite{kalantzis2023matrix}. An algorithm to mitigate the error accumulated  by a sequence of incremental updates, is proposed in \cite{zhang2017timers} which uses a theoretical error bound to trigger a full Singular Value Decomposition (SVD) restart only when the accumulated error exceeds a pre-defined threshold, ensuring a balance between computational efficiency and spectral accuracy. 

The current work builds on the aforementioned first-order  analysis techniques and advances them by introducing a projection-based algorithm that attains high-quality approximations of the top-$K$ eigenpairs of an evolving adjacency matrix, while enjoying low computational complexity and memory requirements.

The remainder of this section provides additional details regarding first-order analysis methods.

\subsection{Updating embeddings using perturbation-based techniques}
\label{ssec:perturbation}

First-order analysis methods update the leading eigenpairs of $\mA$ by constructing quadratic-error approximations of $\widehat{\lambda}_j$ and $\widehat{\vx}_j$.
Under the assumption that $\lambda_j$ 
is simple, a first-order sensitivity analysis yields \cite{byron1992mathematics, kato2013perturbation, wilkinson1965algebraic,rellich1969perturbation, stewart1990matrix}:
\begin{equation} \label{eq:pert}
\begin{aligned}
    \widehat{\lambda}_j & = \lambda_j + \overline{\mathbf{x}}_j^\top \mathbf{\Delta \overline{\mathbf{x}}}_j + O(\|\mathbf{\Delta}\|^2),\\
    \mathbf{\widehat{x}}_j & = \overline{\mathbf{x}}_j +\sum_{i=1,i\neq j}^{i=N} \dfrac{\overline{\mathbf{x}}_i^\top\mathbf{\Delta}\overline{\mathbf{x}}_j}{\lambda_j-\lambda_i}\overline{\mathbf{x}}_i + O(\|\mathbf{\Delta}\|^2).
\end{aligned}
\end{equation}
Extensions for the case of repeating eigenvalues can be found in, e.g., \cite{byron1992mathematics}. Next, three algorithms that approximate 
the $K$ leading eigenpairs of the matrix $\widehat{\mA}$ leveraging 
first-order eigenvalue/eigenvector expansions are introduced. Henceforth, these approaches will be referred to as \emph{perturbation-based} techniques.

\subsubsection{TRacking eIgenPairs Basic (TRIP-Basic)}

While the eigenvalue update $\widehat{\lambda}_j$ in (\ref{eq:pert}) depends solely on $(\lambda_j, \mathbf{\overline{x}}_j)$, the corresponding eigenvector update $\mathbf{\widehat{x}}_j$ requires all $N$ eigenpairs of $\mathbf{\overline{A}}$. Nonetheless, only the $K$ leading eigenpairs of $\mathbf{\overline{A}}$ are available.  TRIP-Basic \cite{chen2015fast} approximates $\widehat{\vx}_j$ via the $K$ eigenpairs 
$\{({\lambda}_{j}, {\mathbf{\overline{x}}}_{j})\}_{j=1}^{K}$ . This yields $(\widetilde{\lambda}_j, \widetilde{\vx}_j)$ as the estimate of $j$th eigenpair of $\widehat{\mA}$ where:
\begin{align}
    \label{eq:trip-basic-eigvec}
    \widetilde{\lambda}_j & = \lambda_j + \overline{\mathbf{x}}_j^\top 
    \mathbf{\Delta \overline{\mathbf{x}}}_j\\
    \mathbf{\widetilde{x}}_j  &= \mathbf{x}_j +\sum_{i=1,i\neq j}^{i=K} \dfrac{\overline{\mathbf{x}}_i^\top\mathbf{\Delta}\overline{\mathbf{x}}_j}{\lambda_j-\lambda_i}\overline{\mathbf{x}}_i,
\end{align}

\noindent TRIP-Basic incurs computational complexity of $O(K(N + {\rm nnz}(\mathbf{\Delta})))$ and memory complexity of $O(NK + {\rm nnz}(\mathbf{\Delta}))$.

\subsubsection{TRacking eIgenPairs (TRIP)}
The TRackIng eigenPairs (TRIP) algorithm extends TRIP-Basic 
by delaying the computation of $\mathbf{\widetilde{x}}_j$ 
until $\widetilde{\lambda}_j$ becomes available \cite{chen2015fast}. 
Let the approximate eigenvector be written as 
$\mathbf{\widetilde{x}}_j = \overline{\mathbf{X}}_K \mathbf{b}_j = \sum_{i=1}^{i=K}[\mathbf{b}_j]_i
\overline{\mathbf{x}}_i$, and define the $K\times K$ diagonal: 
\begin{equation*}
    \mathbf{W}_j=\mathrm{diag}(\tilde{\lambda}_j-\lambda_1,\ldots,\tilde{\lambda}_j-\lambda_K),
\end{equation*}
with $\tilde{\lambda}_j$  defined as in Eq.~\eqref{eq:trip-basic-eigvec}.
Similarly, let 
$\mathbf{\overline{X}}_K
=\left[\mathbf{\overline{x}}_1,\ldots,\mathbf{\overline{x}}_K\right]$
denote the matrix formed by the columns of the $K$ leading eigenvectors of $\mathbf{\overline{A}}$. 
The TRIP algorithm computes the vector $\mathbf{b}_j$ by solving 
the following $K\times K$ system of linear equations:
\begin{equation}\label{eq:beta}
        (\mathbf{W}_j-\overline{\mathbf{X}}_K^\top \mathbf{\Delta}\overline{\mathbf{X}}_K)\mathbf{b}_j=\overline{\mathbf{X}}_K^\top \mathbf{\Delta}\overline{\mathbf{x}}_j.
\end{equation}

\noindent TRIP incurs a computational complexity of $O(K^3 + K(N+{\rm nnz}(\mathbf{\Delta})))$, and a memory complexity of $O(NK + K^2 + {\rm nnz}(\mathbf{\Delta}))$. 

\subsubsection{Updating with Residual Modes}

A major limitation of TRIP-Basic and TRIP is the lack of a correction mechanism for the part of $\mathbf{\widehat{x}}_j$ that is associated with the non-computed eigenvectors $\mathbf{x}_{K+1},\ldots,\mathbf{x}_N$ of $\mathbf{A}$. A technique to achieve 
the latter under the same asymptotic complexity is to introduce residual modes \cite{mitz2022perturbation}. Expressing \eqref{eq:pert} up to the $O(\|\mathbf{\Delta}\|^2)$ term:
\begin{align}
\mathbf{\widehat{x}}_j -\overline{\mathbf{x}}_j =
\sum_{i=1,i\neq j}^{K} \dfrac{\overline{\mathbf{x}}_i^\top\mathbf{\Delta}\overline{\mathbf{x}}_j}{\lambda_j-\lambda_i}\overline{\mathbf{x}}_i+\sum_{i=K+1}^{N} \dfrac{\overline{\mathbf{x}}_i^\top\mathbf{\Delta}\overline{\mathbf{x}}_j}{\lambda_j-\lambda_i}\overline{\mathbf{x}}_i.   \label{eq:resmode_ideal} 
\end{align}
Replacing the unknown eigenvalues $\lambda_{K+1},\ldots,\lambda_N$ by some scalar $\mu\in \mathbb{R}$ allows to extract a common multiplier:
\begin{equation*}
    \sum_{i=K+1}^{N} \dfrac{\overline{\mathbf{x}}_i^\top\mathbf{\Delta}\overline{\mathbf{x}}_j}{\lambda_j-\mu}\overline{\mathbf{x}}_i = 
    \dfrac{1}{\lambda_j-\mu}\sum_{i=K+1}^{N} \overline{\mathbf{x}}_i\overline{\mathbf{x}}_i^\top (\mathbf{\Delta}\overline{\mathbf{x}}_j),
\end{equation*}
where the orthogonality of the eigenvectors $\mathbf{x}_1,\ldots,\mathbf{x}_N$, 
yields:
\begin{align*}
\sum_{i=K+1}^{N} \overline{\mathbf{x}}_i\overline{\mathbf{x}}_i^\top (\mathbf{\Delta}\overline{\mathbf{x}}_j) & =
    (\overline{\mathbf{x}}_{K+1}\overline{\mathbf{x}}_{K+1}^\top+\ldots +\overline{\mathbf{x}}_N\overline{\mathbf{x}}_N^\top)\mathbf{\Delta}\overline{\mathbf{x}}_j\\ &= 
    (\mathbf{I}-\overline{\mathbf{X}}_K\overline{\mathbf{X}}_K^\top)\mathbf{\Delta}\overline{\mathbf{x}}_j,
\end{align*}
Thus, \eqref{eq:resmode_ideal} can be expressed as:
\begin{equation*}\label{eq:gamma}
    \mathbf{\widetilde{x}}_j  = \overline{\mathbf{x}}_j+ \sum_{i=1,i\neq j}^{K} \dfrac{\overline{\mathbf{x}}_i^\top\mathbf{\Delta}\overline{\mathbf{x}}_j}{\lambda_j-\lambda_i}\overline{\mathbf{x}}_i + \dfrac{1}{\lambda_j-\mu}(\mathbf{I}-\overline{\mathbf{X}}_K\overline{\mathbf{X}}_K^\top)\mathbf{\Delta}\overline{\mathbf{x}}_j.
\end{equation*}

\noindent Residual modes incur a computational complexity of $O(K(N+{\rm nnz}(\mathbf{\Delta})))$, and a memory complexity of $O(NK + {\rm nnz}(\mathbf{\Delta}))$. 

\subsection{Adding graph nodes} 
\label{ssec:adding_nodes}

The techniques discussed so far can provide a fast approximation of the updated 
eigenpairs $(\widehat{\lambda}_j,\mathbf{\widehat{x}}_j),\ j=1,\ldots,K$. 
However, as shown in the next proposition, these techniques might not be suitable 
when the graph ${\cal G}$ is expanded with $S>0$ nodes.

\begin{proposition}\label{prom1} 
Consider the setup of Problem \ref{prob:st-update}. For any of the approaches outlined in Section \ref{ssec:perturbation}, the approximation $(\widetilde{\lambda}_j,\widetilde{\mathbf{x}}_j)$ of $(\widehat{\lambda}_j,\widehat{\mathbf{x}}_j)$ disregards the sub-matrix $\mathbf{C}$ of $\mDelta$. 
\end{proposition}
\begin{proof}
Following any of the approaches outlined in Section \ref{ssec:perturbation}, the eigenvalue $\lambda_j$ is updated via the formula $\widetilde{\lambda}_j = \lambda_j + \overline{\mathbf{x}}_j^\top \mathbf{\Delta \overline{x}}_j$. However, note that for any 
$1\leq i,j\leq K$: 
\begin{equation*}
\mathbf{\overline{x}}_i^\top\mathbf{\Delta}\mathbf{\overline{x}}_j = 
\begin{bmatrix} \mathbf{x}_i^\top &
\mathbf{0}_{S,1}^\top\end{bmatrix}\begin{bmatrix}\mathbf{K} & \mathbf{G}\\ \mathbf{G}^\top & \mathbf{C}
\end{bmatrix}
\begin{bmatrix} \mathbf{x}_j\\
\mathbf{0}_{S,1}\end{bmatrix}=\mathbf{x}_i^\top\mathbf{K}\mathbf{x}_j,
\end{equation*}
and thus $\widetilde{\lambda}_j = \lambda_j+\mathbf{x}_i^\top\mathbf{K}\mathbf{x}_i$. 
Then, the residual mode of the $j$th updated eigenvector 
is equal to: 
\begin{align*}
    \left(\mathbf{I}-\mathbf{\overline{X}}_K\mathbf{\overline{X}}_K^\top\right)\mathbf{\Delta}\mathbf{\overline{x}}_j &= 
    \begin{bmatrix}
        \mathbf{I}-\mathbf{X}_K\mathbf{X}_K^\top & \mathbf{0}_{N,S}\\
        \mathbf{0}_{S,N} & \mathbf{I}_{S,S}\\
    \end{bmatrix}
    \begin{bmatrix}
\mathbf{K}\mathbf{x}_j \\
\mathbf{G}^\top\mathbf{x}_j\\
    \end{bmatrix}\\ &=
    \begin{bmatrix}
\left(\mathbf{I}-\mathbf{X}_K\mathbf{X}_K^\top\right)\mathbf{K}\mathbf{x}_j \\
\mathbf{G}^\top\mathbf{x}_j\\
    \end{bmatrix}.
\end{align*}
\end{proof}
\begin{corollary}
 Consider the same parameters as in Proposition \ref{prom1} and let $\mathbf{K}=\mathbf{0}_{N,N}$. 
 Then, $\widetilde{\lambda}_j={\lambda}_j$. 
\end{corollary}
\begin{proof}
It is straightforward to verify that for any 
$1\leq i,j\leq K$: 
\begin{equation*}
\mathbf{\overline{x}}_i^\top\mathbf{\Delta}\mathbf{\overline{x}}_j = 
\begin{bmatrix} \mathbf{x}_i^\top &
\mathbf{0}_{S,1}^\top\end{bmatrix}\begin{bmatrix}\mathbf{0}_{N,N} & \mathbf{G}\\ \mathbf{G}^\top & \mathbf{C}
\end{bmatrix}
\begin{bmatrix} \mathbf{x}_j\\
\mathbf{0}_{S,1}\end{bmatrix}=0,
\end{equation*}
and thus $\widetilde{\lambda}_j = \lambda_j$.
\end{proof}
Thus, perturbation-based techniques such as those presented 
in Section \ref{ssec:perturbation} fall short when the graph ${\cal G}$ is dynamically 
augmented with new nodes.

\section{Updating the eigenspace via Matrix Projections}
This section gradually introduces the proposed eigenpair update scheme for a single time-step update. First, the algorithms of Sec.~\ref{ssec:perturbation} are presented under a novel unified viewpoint. Following that, this viewpoint is extended to capture the cases of new node arrivals.
\label{sec:projections_main}

\subsection{A unifying viewpoint of perturbation-based approaches}

In Sec.~\ref{ssec:perturbation}, the approximation of $\widehat{\mathbf{x}}_j$ is computed via a matrix-vector multiplication 
between a basis matrix and a vector of scalar coefficients, i.e., for a basis 
matrix $\mathbf{Z}\in \mathbb{R}^{N\times D}$, the approximate eigenvectors are computed as
$\widehat{\mathbf{x}}_j=\mathbf{Z}\mathbf{v}_j,\ \mathbf{v}_j\in \mathbb{R}^{D\times 1},\ j=1,\ldots,K$. 
TRIP-Basic and TRIP approximate the eigenvector $\widehat{\mathbf{x}}_j$ 
via  $\widetilde{\mathbf{x}}_j = \mathbf{\overline{X}}_K \mathbf{a}_j,\ \mathbf{a}_j\in \mathbb{R}^K$ and 
    $\widetilde{\mathbf{x}}_j = \mathbf{\overline{X}}_K \mathbf{b}_j,\ \mathbf{b}_j\in \mathbb{R}^K$, respectively, 
    where:
\begin{equation*}
    [\mathbf{a}_j]_i = \begin{cases}
        1, & \text{if}\ i = j \\
        \dfrac{
            \overline{\mathbf{x}}_i^\top \mathbf{\Delta} \overline{\mathbf{x}}_j
        }{
            \lambda_j - \lambda_i
        }, & \text{otherwise};
\end{cases} 
\end{equation*}
and:
\begin{equation*}
    \mathbf{b}_j = 
        (\mathbf{W}_j - \mathbf{\overline{X}}_k^\top\mathbf{\Delta}\mathbf{\overline{X}}_k)^{-1} \mathbf{\overline{X}}_k^\top \mathbf{\Delta} \mathbf{\overline{x}}_j.
\end{equation*}
Thus, TRIP-Basic and TRIP are extracting $\widetilde{\mathbf{x}}_j$ from the subspace 
$\Ran(\mathbf{\overline{X}}_K)$, where each basis vector is scaled according to the entries of the vectors $\mathbf{a}_j$ and $\mathbf{b}_j$, respectively. 
Similarly, the addition of residual modes leads to an approximation of the  eigenvector $\widehat{\mathbf{x}}_j$ via $\widetilde{\mathbf{x}}_j = 
    \left[\mathbf{\overline{X}}_K, (\mathbf{I}-\mathbf{\overline{X}}_K\mathbf{\overline{X}}_K^\top)\mathbf{\Delta}\mathbf{\overline{x}}_j\right] \mathbf{c}_j,\ \mathbf{c}_j\in \mathbb{R}^N$, where for any $\mu \in \mathbb{R}$ not equal to $\lambda_j$:
\begin{equation*}
    [\mathbf{c}_j]_i = \begin{cases}
        1, & \text{if}\ i = j \\
        \dfrac{
            \overline{\mathbf{x}}_i^\top \mathbf{\Delta} \overline{\mathbf{x}}_j
        }{
            \lambda_j-\lambda_i
        }, & \text{if}\ 1 < i \leq k \\
        1/(\lambda_j - \mu), & \text{otherwise}.
\end{cases} 
\end{equation*}
Thus, the approximate eigenvector $\widetilde{\mathbf{x}}_j$ is extracted  
from the subspace $\Ran\left(\left[\mathbf{\overline{X}}_K, (\mathbf{I}-\mathbf{\overline{X}}_K\mathbf{\overline{X}}_K^\top)\mathbf{\Delta}
\mathbf{\overline{x}}_j\right]\right)$ where each basis vector is scaled 
according to the entries of the vector $\mathbf{c}_j$. 

Nonetheless, there may exist linear combinations that lead to improved approximation accuracy compared to pre-computing the coefficients $\mathbf{v}$, that is, given a fixed basis $\mathbf{Z}$, there may exist $\mathbf{v}^*$ such that:
\begin{equation*}
\|\widehat{\mathbf{A}}\mathbf{Z}\mathbf{v}^* - \lambda\mathbf{Z}\mathbf{v}^*\| < \|\widehat{\mathbf{A}}\mathbf{Z}\mathbf{v} - \lambda\mathbf{Z}\mathbf{v}\|.
\end{equation*}
The ensuing subsection summarizes a projection scheme that derives the 
basis coefficients following an optimality criterion instead of the analytic formulas 
discussed so far.

\subsection{Rayleigh-Ritz projections}
\label{ssec:RR}
 
The Rayleigh-Ritz (RR) projection method extracts approximations of the 
algebraically smallest and largest eigenvalues, and their corresponding 
eigenvectors, of  $\widehat{\mA}$ from an ansatz subspace 
${\cal Z}\subseteq \mathbb{R}^N$ \cite{golub2013matrix}. Specifically, 
let $\mathbf{Z}\in \mathbb{R}^{N\times D}$ denote an orthonormal basis 
of the subspace ${\cal Z}$ where $D\in \mathbb{N},\ D>K$, denotes the 
corresponding subspace dimension, and 
$\{(\theta_{j},\mathbf{f}_{j})\}_{j=1}^{K}$ denote the 
$K$ leading eigenpairs of the matrix $\mathbf{Z}^\top \widehat{\mA}\mathbf{Z}$ 
such that $\theta_{1}\geq\theta_{2}\geq\cdots \geq\theta_{K}$. Then, the $j$th leading 
eigenpair of the matrix $\widehat{\mA}$ can be approximated by so-called the $j$th leading 
Ritz pair $(\theta_{j},\mathbf{Z}\mathbf{f}_{j})$. This procedure 
has an asymptotic complexity of $O(D^3+ND^2)$ and is outlined in 
Alg.~\ref{alg:rayleigh-ritz}.
\begin{algorithm} 
\caption{The Rayleigh-Ritz procedure \cite{golub2013matrix}.}
\label{alg:rayleigh-ritz}
\begin{algorithmic}[1]
\Require{Matrix $\widehat{\mA} \in \mathbb{R}^{N\times N}$, orthonormal basis $\mathbf{Z} \in \mathbb{R}^{N\times D}$}
\Ensure{Ritz pairs: $\{(\theta_j, \mathbf{Z}\mathbf{f}_{j})\}_{j=1}^K$}
\State Compute the leading eigenpairs 
$\{(\theta_j, \mathbf{f}_{j})\}_{j=1}^K$ of $\mathbf{Z}^\top\widehat{\mA}\mathbf{Z}$\label{step:rrmat}
\State Form the $K$ leading Ritz pairs $\{( \theta_j, \mathbf{Z}\mathbf{f}_{j})\}_{j=1}^K$
\end{algorithmic}
\end{algorithm}

The main idea behind the RR procedure is to extract approximate eigenpairs 
of the matrix $\widehat{\mA}$ from a subspace ${\cal Z}$ which, ideally, 
encapsulates the invariant subspace associated with $K$ leading eigenvectors $\widehat{\mathbf{x}}_1,\ldots,\widehat{\mathbf{x}}_K$. The Rayleigh-Ritz 
projection computes the coefficients that multiply the basis vectors by 
imposing the Ritz-Galerkin condition 
$\mathbf{Z}^\top\left(\widehat{\mA}\mathbf{Z}\mathbf{f}-\theta \mathbf{Z}\mathbf{f}\right)=0$. Despite its simplicity, in the ensuing theorem, 
the RR procedure is shown to extract the optimal eigenvector approximations 
from the subspace ${\cal Z}$. 
\begin{theorem}[Theorem~7.1 \cite{demmel1997applied}] \label{thm1}
The minimum of $\|\widehat{\mA}\mathbf{Z}-\mathbf{Z}\mathbf{S}\|_2$ over all $d\times d$ matrices
$\mathbf{S}$ is achieved when $\mathbf{S}=\mathbf{\mathbf{Z}^\top \widehat{\mA}\mathbf{Z}}$ in which case $\|\widehat{\mA}\mathbf{Z}-\mathbf{Z}\mathbf{S}\|_2 = \|\mathbf{Z}^\top \widehat{\mA}\mathbf{Z}_\bot\|_2$ where $[\mathbf{Z},\mathbf{Z}_{\bot}]$ is orthonormal. Writing 
$\mathbf{\mathbf{Z}^\top \widehat{\mA}\mathbf{Z}}=\mathbf{F}\mathbf{\Theta}\mathbf{F}^\top,\ \mathbf{F}^\top\mathbf{F}=\mathbf{I}$, the minimum $\|\widehat{\mA}\mathbf{P}-\mathbf{P}\mathbf{D}\|$ over all $N\times D$ orthogonal matrices $\mathbf{P}$ with $\Ran(\mathbf{P})=\Ran(\mathbf{Z})$ 
and $D\times D$ diagonal matrices $\mathbf{D}$ is also $\|\mathbf{Z}^\top \widehat{\mA}\mathbf{Z}_\bot\|_2$ 
and is achieved when $\mathbf{P}=\mathbf{Z}\mathbf{F}$ and $\mathbf{D}=\mathbf{\Theta}$.
\end{theorem}

Theorem \ref{thm1} states that the columns of the matrix $\mathbf{Z}\mathbf{F}$ are the 
“optimal” approximate eigenvectors and the diagonal entries of $\mathbf{\Theta}$ are
the “optimal” approximate eigenvalues in the sense of minimizing the residual
$\|\widehat{\mA}\mathbf{P}-\mathbf{P}\mathbf{D}\|$.  
Thus, the choice of projection subspace ${\cal Z}$ and corresponding projection basis $\mathbf{Z}$ is critical for the approximation quality and computational complexity of the RR procedure. 

Following the above discussion, a natural choice for the projection subspace 
is either ${\cal Z} = \Ran(\mathbf{X}_K)$ or ${\cal Z} = \Ran\left(\left[\mathbf{\overline{X}}_K,(\mathbf{I}-\mathbf{\overline{X}}_K\mathbf{\overline{X}}_K^\top)\mathbf{\Delta}\mathbf{\overline{X}}_K\right]\right)$. While these subspaces are the same as in TRIP and Residual Modes respectively, using them in the RR procedure is expected to yield improved eigenvector approximation, as RR seeks the optimal coefficients.
Given a basis matrix $\mathbf{Z}$ corresponding to the aforementioned subspaces $\cal Z$, the eigenpairs of $\widehat{\mA}$ can be approximated via Alg.~\ref{alg:rayleigh-ritz}. Specifically, line~\ref{step:rrmat} of Alg.~\ref{alg:rayleigh-ritz} computes the 
$K$ dominant eigenpairs of:
\begin{equation}
    \mathbf{Z}^{\top}\widehat{\mA}\mathbf{Z} = \mathbf{Z}^{\top}\overline{\mA}\mathbf{Z} + \mathbf{Z}^{\top}\mDelta\mathbf{Z}.
\end{equation}
Nonetheless, the discussion in Sec.~\ref{ssec:adding_nodes}, and  
Proposition \ref{prom1}, indicate that using only ${\cal Z} = 
\Ran(\mathbf{X}_K)$ leads to parts of the update matrix being ignored, 
especially when $\mA$ has smaller size than $\mathbf{\widehat{A}}$.

\subsection{Constructing the projection subspace for expanding graphs}
\label{ssec:proj_subspace}

In this section an enhancement of the Rayleigh-Ritz projection subspace is proposed,
via re-visiting the structure of 
$(\mathbf{I}-\overline{\mathbf{X}}_K\overline{\mathbf{X}}_K^\top)\mathbf{\Delta}\overline{\mathbf{X}}_K$ in greater detail, focusing on the matrix product $\mathbf{\Delta}\overline{\mathbf{X}}_K$. 

\begin{proposition}
\label{pro3}
Partition $\mathbf{\Delta} = [\mathbf{\Delta}_1,\mathbf{\Delta}_2]$ where 
$\mathbf{\Delta}_1\in \mathbb{R}^{(N+S)\times N}$ and $\mathbf{\Delta}_2\in \mathbb{R}^{(N+S)\times S}$ denote the rectangular submatrices formed by the 
$N$ leading and $S$ trailing columns of $\mathbf{\Delta}$, respectively. Then: 
\begin{equation*}
\Ran\left(\mathbf{\Delta}\overline{\mathbf{X}}_K\right) \subseteq \Ran\left(\mathbf{\Delta}_1\right).
\end{equation*}
\end{proposition}
\begin{proof}
The matrix product 
between $\mathbf{\Delta}$ and $\overline{\mathbf{X}}_k$ can be written as:
\begin{equation}\label{mat1}
\begin{aligned}
\mathbf{\Delta}\overline{\mathbf{X}}_K &= \left[\mathbf{\Delta}_1,\mathbf{\Delta}_2\right]
\begin{bmatrix}
\mathbf{X}_K\\
\mathbf{0}_{S, K}
\end{bmatrix}= \mathbf{\Delta}_1\mathbf{X}_K.
\end{aligned}
\end{equation}
Thus, all columns of $\mathbf{\Delta}\overline{\mathbf{X}}_K$ are 
linear combinations of the columns of $\mathbf{\Delta}_1$.
\end{proof}

Proposition \ref{pro3} shows that 
$\mathbf{\Delta}\overline{\mathbf{X}}_K$ does not 
include any linear combinations involving the trailing 
$S$ columns of  $\mathbf{\Delta}$. In turn, 
this implies that the range \smash{$\Ran\left(\left[\mathbf{\overline{X}}_K,(\mathbf{I}-\mathbf{\overline{X}}_K\mathbf{\overline{X}}_K^\top)
\mathbf{\Delta}\mathbf{\overline{X}}_K\right]\right)$}
will always lack  information included in $(\mathbf{I}-\mathbf{\overline{X}}_K\mathbf{\overline{X}}_K^\top)\mathbf{\Delta}_2$. As a result, the latter subspace needs to 
be added explicitly to the RR projection subspace, that is:
\begin{equation}
    {\cal Z} = \Ran\left(\left[
        \overline{\mX}_K, (\mI - \overline{\mX}_K \overline{\mX}_K^\top) [\mDelta \overline{\mX}_K, \mDelta_2]
    \right]\right).
\end{equation}

\begin{table}
\centering
\caption{{Subspaces corresponding to different algorithms for approximating   $\mathbf{\widetilde{x}}_j$.}} \label{tab:subspaces}
\begin{tabular}{ l c c}
  \toprule
  \# Algorithm & ${\cal Z}$ \\
  \midrule
  1. TRIP-Basic &   $\Ran(\mathbf{\overline{X}}_K\mathbf{a}_j)$  \\
   2. TRIP &   $\Ran(\mathbf{\overline{X}}_K\mathbf{b}_j)$  \\
   3. Residual modes &   $\Ran\left(\left[\mathbf{\overline{X}}_K,(\mathbf{I}-\mathbf{\overline{X}}_K\mathbf{\overline{X}}_K^\top)\mathbf{\Delta}\mathbf{\overline{x}}_j\right]\mathbf{c}_j\right)$  \\
   4. Proposed  & $\Ran\left(\left[\mathbf{\overline{X}}_K,(\mathbf{I}-\mathbf{\overline{X}}_K\mathbf{\overline{X}}_K^\top)\left[\mathbf{\Delta}\mathbf{\overline{X}}_K,\mathbf{\Delta}_2\right]\right]\right)$\\
\bottomrule
\end{tabular}
\end{table}
Table \ref{tab:subspaces} summarizes the subspace from which different algorithms 
approximate the eigenvector $\widehat{\mathbf{x}}_j$ of  $\widehat{\mathbf{A}}$. Classical perturbation-based techniques approximate $\widehat{\mathbf{x}}_j$ from a one-dimensional subspace defined by a specific linear combination of the columns of the matrix $\mathbf{\overline{X}}_K$, and each separate eigenvector approximation requires a different linear combination. On the other hand, 
the proposed approach combined with the Rayleigh-Ritz procedure, approximates all $K$ sought eigenpairs of 
$\widehat{\mathbf{A}}$ from the same projection subspace and this subspace can be 
extended with information beyond $\mathbf{\overline{X}}_K$ as can be seen by the last 
entry in Table \ref{tab:subspaces}.

\subsection{Low rank approximation of expanded basis}

Augmenting the RR projection basis by $\left(\mathbf{I}-\mathbf{\overline{X}}_K\mathbf{\overline{X}}_K^\top\right)\mathbf{\Delta}_2$ 
can increase its dimension by $S$. For very large values of $S$ this can become 
impractical, since the cost of the RR projection scales cubically with respect to 
$S$. Nonetheless, as the next proposition shows, there exist scenarios where 
the rank of the matrix $\mathbf{\Delta}_2$ is much smaller than $S$, e.g., 
when the $S$ newly added vertices are sparsely connected with each other, and, at the 
same time, connected only to a few nodes of ${\cal G}$. 

\begin{proposition} \label{pro4}
Let ${\cal C}={\cal \widehat{V}} \setminus {\cal V}$ denote the $S$ 
newly added vertices in the graph ${\cal G}$, $J\in \mathbb{N}$ 
denote the number of vertices of ${\cal G}$ that are connected to 
${\cal C}$, and $Q\in \mathbb{N}$ denote the number of vertices of 
${\cal C}$ that are connected to ${\cal G}$. Then, $\Rank\left(\mathbf{\Delta}_2\right)\leq \min(J,Q)$. 
\end{proposition}
\begin{proof}
The rank of any matrix can not exceed its minimum number of non-zero 
rows and columns. Since $\left[\mathbf{\Delta}_2\right]_{i,j}$ 
is non-zero if and only if the $i$th vertex of ${\cal G}$ is connected 
to the $j$th vertex of the set ${\cal C}$, $\mathbf{\Delta}_2$ has 
at most $J$ non-zero rows and $Q$ non-zero columns. Therefore, a basis 
for the column space of $(\mathbf{I}-\mathbf{\overline{X}}_K\mathbf{\overline{X}}_K^\top)\mathbf{\Delta}_2$ 
can be readily computed from its $\min(J,Q)$ left singular vectors.
\end{proof}

\subsection{Decreasing complexity via Randomized SVD}
\label{ssec:RSVD}

When $(\mathbf{I} -\overline{\mX}_K \overline{\mX}_K^\top) \mathbf{\Delta}_2$ 
is not explicitly low-rank, an approximate basis of its column space can be 
computed via leveraging Randomized SVD (RSVD) \cite{martinsson2020randomized,halko2011finding}. Suppose we 
are interested in computing $L\in \mathbb{N}$ leading left singular vectors of 
$(\mathbf{I}- \mathbf{\overline{X}}_K\mathbf{\overline{X}}_K^\top)
\mathbf{\Delta}_2$. The steps of the RSVD procedure are outlined below: 
\begin{enumerate}[label=\textbf{S.\arabic*}]
    \item Given a so-called oversampling parameter $P$, the product $\mathbf{Y} = (\mathbf{I} - \mathbf{\overline{X}}_K \mathbf{\overline{X}}_K^\top) \mathbf{\Delta}_2 \mathbf{\Omega}$ is computed, where $\mathbf{\Omega}$ is a $S\times (L+P)$ matrix with i.i.d. $\mathcal{N}(0,1)$ Gaussian entries; \label{enumitem:rsvd:s1}
    \item The $L$ leading singular vectors of $\mathbf{M}^\top (\mathbf{I} -\mathbf{\overline{X}}_K \mathbf{\overline{X}}_K^\top) \mathbf{\Delta}_2 = \widehat{\mathbf{U}} \widehat{\mathbf{\Sigma}} \widehat{\mathbf{V}}^{\top}$ are computed, where $\mathbf{M}$ is an orthonormal basis of $\Ran(\mathbf{Y})$; \label{enumitem:rsvd:s2}
    \item The rank-$L$ approximation of $(\mathbf{I} -\mathbf{\overline{X}}_K \mathbf{\overline{X}}_K^\top) \mathbf{\Delta}_2$ is set as $\mathbf{M}\widehat{\mathbf{U}}\widehat{\mathbf{\Sigma}}\widehat{\mathbf{V}}^{\top}$; and
    \item Since $\mathbf{\widehat{U}}^\top\mathbf{M}^\top\mathbf{M}\mathbf{\widehat{U}}=\mathbf{I}$, the orthogonal matrix $\mathbf{R} = \mathbf{M}\widehat{\mathbf{U}}$ forms an approximation of the $L$ left leading singular vectors of $(\mathbf{I} - \mathbf{\overline{X}}_K \mathbf{\overline{X}}_K^\top) \mathbf{\Delta}_2$. \label{enumitem:rsvd:s4}
\end{enumerate}
Step \ref{enumitem:rsvd:s1} incurs a complexity of $\ O(2NK(L+P))$ while the formation of the basis $\mathbf{M}$ in step \ref{enumitem:rsvd:s2} incurs a complexity of $O(N(L+P)^2)$. Finally, forming the singular vector approximation $\mathbf{R}$ in step \ref{enumitem:rsvd:s4} incurs complexity $O(NL(L+P))$. Note, that if the rank of $\mathbf{\Delta}_2$ is less than $L+P$, then with probability one RSVD returns the range space $\Ran\left((\mathbf{I} -\mathbf{\overline{X}}_K \mathbf{\overline{X}}_K^\top) \mathbf{\Delta}_2\right)$ \cite{martinsson2020randomized}. Nevertheless, $L$ and $P$ can be user-defined and induce a trade-off between computational complexity and approximation accuracy, allowing practitioners to select them according to their accuracy and runtime requirements. This RSVD step can be especially beneficial when a large number of nodes are added to the graph, i.e., when $S$ is large.

\section{Dynamic graph updates}
\label{sec:main_graph_updates}

The previous section presented a subspace projection technique to update the $K$ leading eigenpairs of an adjacency matrix subject to vertex and edge modifications, for a single time-step. We now turn our attention to the problem of dynamic graph updating over $T\in \mathbb{N}$ time-steps starting from an initial graph ${\cal G}^{(0)}$ and updating $\{{\cal G}^{(t)}\}_{t=0}^{T}$ with $\calG^{(t)} = (\calV^{(t)}, \calE^{(t)})$ denoting the graph observed at time-step $t=0,\ldots,T$. The adjacency matrix associated with the graph ${\cal G}^{(t+1)}$ obeys the transition model outlined in 
\eqref{eq:adj_update}.

The proposed algorithm is tabulated in Alg.~\ref{alg:dynamic_graph_updates} 
and is abbreviated as G-REST (Graph Rayleigh-Ritz Eigenspace Tracking).
In the following, the leading $K$ eigenvalues of $\mA^{(t)}$ are collected in the $K\times K$ diagonal matrix $\mLambda^{(t)}$ and the corresponding eigenvectors are collected as columns of the $N^{(t)}\times K$ matrix $\mX_{K}^{(t)}$. The $K$ leading eigenpairs of adjacency $\mathbf{A}^{(0)}$ of the initial graph ${\cal G}^{(0)}$, $\{(\lambda_{i}^{(0)},\vx_{i}^{(0)})\}_{i=1}^{K}$, are computed using any direct eigendecomposition method. 
At time step $t+1$, the matrix update $\mDelta^{(t+1)}$ with $S^{(t+1)}$ new nodes is received, with ${\mDelta_2^{(t+1)}}$ denoting the trailing $S^{(t+1)}$ columns of $\mDelta^{(t+1)}$, and the eigenvector matrix $\mX_{K}^{(t)}$ from the time-step $t$ is padded with a $S^{(t+1)}\times K$ matrix of zero entries, i.e. 
$\overline{\mX}_{K}^{(t)} = \begin{bmatrix}
    \mX_{K}^{(t)} \\ \mathbf{0}_{S^{(t+1)},K}
\end{bmatrix}$. 

The key step during iteration $t+1$ is the selection of the projection subspace ${\cal Z}^{(t+1)}$ and the construction of the Rayleigh-Ritz projection basis $\mathbf{Z}^{(t+1)}$ in line \ref{keystep}. As mentioned in Sec.~\ref{sec:projections_main} and outlined in Table~\ref{tab:subspaces}, options for ${\cal Z}^{(t+1)}$ include $\Ran(\overline{\mX}_K^{(t)})$ from TRIP, \[\Ran\left(\left[\mathbf{\overline{X}}_K^{(t)},(\mathbf{I}-\mathbf{\overline{X}}_K^{(t)}{\mathbf{\overline{X}}_K^{(t)}}^\top)\mathbf{\Delta}^{(t+1)}\mathbf{\overline{X}}_K^{(t)}\right]\right)\] from Residual modes, and the proposed subspace of Sec.~\ref{ssec:proj_subspace} \[\Ran\left(\left[\mathbf{\overline{X}}_K^{(t)},(\mathbf{I}-\mathbf{\overline{X}}_K^{(t)}{\mathbf{\overline{X}}_K^{(t)}}^\top)\left[\mathbf{\Delta}^{(t+1)}\mathbf{\overline{X}}_K^{(t)},\mathbf{\Delta}_2^{(t+1)}\right]\right]\right).\] In all cases, the basis $\mathbf{Z}^{(t+1)}$ can be constructed via retaining the 
orthogonal matrix factor of the QR decomposition \cite{golub2013matrix}.  
Notice that, for the proposed subspace, an exact basis can be obtained via the QR decomposition on the 
matrix $(\mathbf{I}-\mathbf{\overline{X}}_{K}^{(t)}{\mathbf{{\overline{X}}}_{K}^{(t)}}^\top)\mDelta_2^{(t+1)}$, however this decomposition can be considerably expensive to obtain \cite{golub2013matrix}. To maintain tractable computational complexity, the RSVD procedure of Section \ref{ssec:RSVD} can be used 
to obtain an approximate basis $\mathbf{R}^{(t+1)}$ of $\mDelta_2^{(t+1)}$. Then, the computational process to obtain an orthonormal $\mathbf{Z}^{(t+1)}$ can be simplified due to the partially orthogonal structure of the subspace 
$\Ran\left(\left[\mathbf{\overline{X}}_{K}^{(t)}, \left(\mathbf{I}-{\mathbf{\overline{X}}}_{K}^{(t)}{\mathbf{{\overline{X}}}_{K}^{(t)}}^\top\right)\left[\mathbf{\Delta}^{(t+1)}\mathbf{\overline{X}}_{K}^{(t)},\mathbf{R}^{(t+1)}\right]\right]\right)$. Indeed, since the 
bottom $S^{(t+1)}\times K$ sub-matrix of $\mathbf{\overline{X}}_{K}^{(t)}$ is zero 
and $\mathbf{R^{(t+1)}}$ is orthogonal to $ \Ran\left(\mathbf{\overline{X}}_{K}^{(t)}\right)$, the orthogonalization procedure needs to be carried only for the matrix  \[\left[\left(\mathbf{I}-\mathbf{\overline{X}}_{K}^{(t)}{\mathbf{{\overline{X}}}_{K}^{(t)}}^\top\right)\left[\mathbf{\Delta}^{(t+1)}{\mathbf{\overline{X}}}_{K}^{(t)},\mathbf{R}^{(t+1)}\right]\right], \] which in turn requires $O((N^{(t)}+S^{(t+1)})(K+L)^2)$ floating-point operations.

Finally, the eigenpairs of $\mA^{(t+1)}$ can be approximated via the Rayleigh-Ritz procedure of Algorithm \ref{alg:rayleigh-ritz}, which requires the computation of the eigenvalue decomposition of the matrix:
\begin{equation}
\label{eq:RR_eigen_graph_upd_main2} {\mathbf{Z}^{(t+1)\top}}\overline{\mA}^{(t+1)}\mathbf{Z}^{(t+1)} + {\mathbf{Z}^{(t+1)\top}}\mDelta^{(t+1)}\mathbf{Z}^{(t+1)}.
\end{equation}
However, as it is not desirable to store $\mA^{(t)}$ in memory, $\overline{\mA}^{(t+1)}$  can be approximated via its truncated eigendecomposition, that is \eqref{eq:RR_eigen_graph_upd_main2} becomes:
\begin{equation}
\label{eq:RR_eigen_graph_upd_main3} {\mathbf{Z}^{(t+1)\top}}\overline{\mX}_{K}^{(t)}\mLambda_{K}^{(t)}{\overline{\mX}_{K}^{(t)}}^{\top}\mathbf{Z}^{(t+1)} + {\mathbf{Z}^{(t+1)\top}}\mDelta^{(t+1)}\mathbf{Z}^{(t+1)}.
\end{equation}
The latter renders the per-step memory complexity of Algorithm \ref{alg:dynamic_graph_updates} to $O((2N^{(t)} + L + 1)K + \rm{nnz}(\mDelta^{(t+1)}))$. The computational cost of the Rayleigh-Ritz eigenvalue decomposition runs at $O((K+L)^3)$.

\begin{algorithm} 
\caption{Graph Rayleigh-Ritz Eigenspace Tracking (\proposed).}
\label{alg:dynamic_graph_updates}
\begin{algorithmic}[1]
\State {\bf Input:} $\mathbf{A}^{(0)} \in \mathbb{R}^{N^{(0)}\times N^{(0)}},\ K \in \mathbb{N},\ L \in \mathbb{N}$
\State {\bf Output:} $(\mathbf{{X}}_{K}^{(T)},\mathbf{{\Lambda}}_{K}^{(T)})$
\State Compute the $K$ leading eigenpairs $(\mathbf{X}_{K}^{(0)},\mathbf{\Lambda}_{K}^{(0)})$ of $\mathbf{A}^{(0)}$
\State Set $t=0$
 \Do
    \State Receive matrix update $\mathbf{\Delta}^{(t+1)}$
    \State $\mathbf{\overline{X}}_{K}^{(t)} \gets [{\mathbf{X}_{K}^{(t)}}^{\top}, \mathbf{0}_{K,S^{(t+1)}}]^{\top}$
    \State Construct an orthonormal basis $\mathbf{Z}^{(t+1)}$. \label{keystep}
    \State Compute leading eigenpairs 
    $\{(\theta_{j}^{(t+1)},\mathbf{f}_{j}^{(t+1)})\}_{j=1}^{K}$ of \eqref{eq:RR_eigen_graph_upd_main3}
    via a direct eigenvalue solver.\label{keystep2}
    \State Set 
    $\mathbf{X}_{K}^{(t+1)} = \mathbf{Z}^{(t+1)}\mathbf{F}_{K}^{(t+1)}$ and 
    $\mathbf{\Lambda}_{K}^{(t+1)} = \mathbf{\Theta}_{K}^{(t+1)}$
    \State $t \gets t+1$.
 \doWhile{there exist graph updates}
\end{algorithmic}
\end{algorithm}

\subsection{Tracking matrix functions}
Consider a symmetric adjacency matrix $\mA$ with corresponding eigendecomposition $\mA = \mX\mLambda\mX^{\top}$. For a diagonalizable matrix, such as the adjacency, an analytic matrix function is defined as $h(\mA) = \mX h(\mLambda)\mX^{\top},$ where $h(\mLambda)$ applies the function $h$ on every diagonal element of $\mLambda.$ Examples of such functions include matrix polynomials, matrix exponentials $e^{\mA}$ and logarithms $\log(\mA)$ and matrix roots. Using the truncated eigendecomposition of the adjacency $\mA\approx \mX_K\mLambda_K\mX_K^{\top}$, the corresponding matrix function can be approximated as $h(\mA) \approx \hat{h}(\mA) = \mX_Kh(\mLambda_K)\mX_K^{\top}$. Thus the procedure outlined in Alg.~\ref{alg:dynamic_graph_updates} can be readily applied to approximate matrix functions of the adjacency $h(\mA^{(t)})$, as the graph evolves.

\subsection{Special case: Updating the Laplacian}
The aforementioned approach tackled updates to the eigenpairs of the adjacency matrix of an evolving graph. However,  in many cases the eigenpairs of the so-called Laplacian matrix are of interest. 
Consider the adjacency matrix $\mathbf{A}$ of a graph $\cal G$  and the corresponding diagonal degree matrix $\mathbf{D}$, which contains in its diagonal the degrees of every node in $\cal G$, i.e., $\mathbf{D} = {\rm diag}(\mathbf{A}\mathbf{1})$, with $\mathbf{1}$ denoting the $N\times 1$ all-ones vector. The \emph{Laplacian} matrix of $\cal G$ is defined as $\mathbf{L} = \mathbf{D} - \mathbf{A}.$
A simple application of Gershgorin's circle theorem shows that the eigenvalues of $\mathbf{L}$ lie in $[0,2d_{\rm max}],$ with $d_{\rm max}$ denoting the largest node degree. 

To apply Alg.~\ref{alg:dynamic_graph_updates} to the problem of tracking the eigenpairs of an evolving Laplacian matrix $\mathbf{L}^{(t)}$, consider the following shifted Laplacian matrix
\begin{equation}
    \mathbf{T}^{(t)} = \alpha^{(t)}\mathbf{I} - \mathbf{L}^{(t)}.
\end{equation}
With $(\nu_i^{(t)},\bm{\phi}_i^{(t)})$ denoting the $i$th eigenpair of $\mathbf{L}^{(t)}$,  $(\mu_i^{(t)},\bm{\phi}_i^{(t)})$ is an eigenpair of  $\mathbf{T}^{(t)}$ with $\mu_i^{(t)} = \alpha^{(t)} - \nu_i^{(t)}.$ Letting $\alpha^{(t)}  = 2d_{\rm max}^{(t)} $, the leading eigenpairs of $\mathbf{T}^{(t)}$ correspond to the trailing eigenpairs of $\mathbf{L}^{(t)}$. Let $\mathbf{\Delta}_{\mathbf{T}}^{(t)}$ denote the shifted Laplacian matrix updates. Then, Alg.~\ref{alg:dynamic_graph_updates} with $\mathbf{T}^{(t)}$ as input can update the eigenpairs of $\mathbf{T}^{(t)}$, and consequently of $\mathbf{L}^{(t)}.$

A similar procedure can be applied to the normalized Laplacian matrix  $\mathbf{L}_n^{(t)} = \mathbf{I} - \mathbf{D}^{-1/2}\mathbf{A}^{(t)}\mathbf{D}^{-1/2}$~\cite{smola_kondor_graphkernels_2003}. The eigenvalues of $\mathbf{L}_n^{(t)}$ lie in $[0,2]$; therefore, the eigenpairs of $\mathbf{T}_n^{(t)} = 2\mathbf{I} - \mathbf{L}_n^{(t)}$ exhibit the same behavior as the eigenpairs of $\mathbf{T}^{(t)}$. Specifically, the leading eigenpairs of $\mathbf{T}_n^{(t)}$ correspond to the trailing eigenpairs of $\mathbf{L}_n^{(t)}$.

\begin{remark}
    In principle, Alg.~\ref{alg:dynamic_graph_updates} can be adapted to track the trailing eigenpairs of an evolving matrix. However, this would lead to a formulation with significantly higher memory complexity, as step 9 of Alg.~\ref{alg:dynamic_graph_updates} would have to use $\mA^{(t)}$ instead of its rank-$K$ approximation.
\end{remark}

\section{Numerical Tests}
\label{sec:numerical-tests}

\begin{table}[t]
    \centering
    \caption{Datasets considered and their properties. Type S indicates the dataset is static, while D implies the dataset is dynamic.}
    \label{tab:dataset_details}
    \begin{tabular}{p{2cm}c@{\hskip 0.5cm}c@{\hskip 0.5cm}c}
    \toprule
    & \multicolumn{3}{c}{Graph Details} \\ \cmidrule{2-4}
    Dataset & Type & $|\calV|$ & $|{\cal E}|$ \\ \midrule
    Crocodile & S & $11,631$ & $170,773$ \\
    CM-Collab & S & $23,133$ & $93,439$ \\ 
    Epinions & S & $75,879$ & $405,740$ \\ 
    Twitch & S & $168,114$ & $6,797,557$ \\ \midrule
    MathOverflow & D & $24,818$ & $187,986$ \\
    Tech & D & $34,761$ & $107,720$ \\
    Enron & D & $87,273$ & $297,456$ \\
    AskUbuntu & D & $159,316$ & $455,691$ \\
    \bottomrule
    \end{tabular}
\end{table}

In this section, we evaluate the approach proposed in Alg.~\ref{alg:dynamic_graph_updates} on key metrics such as \emph{(i)} accuracy in estimating eigenvectors; \emph{(ii)} run time requirements, and; \emph{(iii)} sensitivity to selection of $L$ and $P$ parameters when RSVD is employed. Moreover, the proposed approach is assessed on two dynamic graph applications: \emph{(iv)} central node identification, and; \emph{(v)} clustering nodes of the graph. 

Throughout our numerical experiments, the proposed approach is denoted as \emph{\proposed} and three variants are considered for evaluation: $\proposed_2$, $\proposed_3$ and $\proposed_{\rm RSVD}$. $\proposed_2$ is Alg.~\ref{alg:dynamic_graph_updates} where the projection matrix $\mathbf{Z}^{(t+1)}$ is an orthonormal basis of $\Ran\left(\left[\mathbf{\overline{X}}_K,(\mathbf{I}-\mathbf{\overline{X}}_K\mathbf{\overline{X}}_K^\top)\mathbf{\Delta}\mathbf{\overline{X}}_K\right]\right)$, which is the same subspace used in the Residual Modes algorithm. $\proposed_3$ is Alg.~\ref{alg:dynamic_graph_updates} where the projection matrix is an orthonormal basis of the subspace defined in \ref{ssec:proj_subspace}. $\proposed_{\rm RSVD}$ uses the same projection subspace as $\proposed_3$, in conjunction with RSVD to reduce computational complexity.  The variants of \emph{\proposed} are compared against four baseline algorithms: \emph{TRIP}~\cite{chen2015fast} and Residual Modes (\emph{RM}) \cite{mitz2019symmetric}, as outlined in Sec.~\ref{ssec:perturbation}, the IASC algorithm of \cite{dhanjal2014efficient}, and \emph{TIMERS}~\cite{zhang2017timers}. IASC is another RR-based method where $\mathbf{Z}^{(t+1)}$ determined via $[\mathbf{X}_{K}^{(t)}, \mathbf{0}; \mathbf{0}, \mathbf{I}]$. TIMERS is a restarting algorithm that triggers a truncated eigendecomposition of $\mA^{(t)}$ whenever a proxy for the eigenvector approximation error exceeds a predefined threshold. Between restart events, TIMERS employs an eigenpair tracking algorithm to update the eigendecomposition. In the following, TIMERS uses IASC for eigenpair tracking.

Unless otherwise stated, the hyperparameter $\mu$ of RM is set to $0$, while the $L$ and $P$ parameters of $\proposed_{\rm RSVD}$ are both set to $100$. The $\theta$ parameter of TIMERS, that controls how often a truncated eigendecomposition is executed, is set to $0.01$. TIMERS is further modified to make sure there are at least $5$ time points between eigendecomposition executions, as it was observed that it can trigger eigendecompositions at all time points in some of our experiments. All tests are conducted on a high-performance computing server provided by Michigan State University. Resource allocation for each run is 8 CPU cores, with 32 GB RAM without any GPU requests. All tests were conducted in MATLAB R2023b~\cite{higham2016matlab} and represent the averages of $10$ Monte Carlo runs\footnote{Implementation of the proposed algorithms and experiment scripts can be accessed at https://github.com/abdkarr/grest.}.

\subsection{Eigenvector Estimation}
\label{ssec:eigs-est-results}

\begin{figure*}[t]
    \centering
    \hspace{2em} \ref{leg:exp1-eigs} \\[0.2em]
%
%
%
\begin{tikzpicture}
\begin{groupplot}[%
    group style={
        group size=4 by 1,
        horizontal sep=0.25cm,
        ylabels at=edge left,
        yticklabels at=edge left,
    },
    width=5.4cm,
    ymajorgrids,
    tick align=outside,
    tick label style={font=\scriptsize},
    tick pos=left,
    ymode=log,
    yticklabel style={xshift=0.5em, yshift=0.2em},
    xlabel={Eigenvector},
    ylabel={Angle $\psi$ (radians)},
    label style={font=\scriptsize},
    ylabel style={yshift=-0.5em},
    xlabel style={yshift=0.25em},
    legend style={
        anchor=center,
        at={(0.5, 0.5)},
        font=\scriptsize,
        legend columns=7,
        /tikz/every even column/.append style={
            column sep=1em,
        }
    },
    legend cell align={left},
    title style={yshift=-0.5em, font=\scriptsize},
    ymin=0.00001,
    ymax=0.3
]

\nextgroupplot[
    title={Crocodile},
    ybar=0.25pt,
    ymode=log,
    xmin=0.5,
    xmax=3.5,
    log origin=infty,
    bar width=4pt,
    xtick={1,2,3},
    xticklabels={{$\vx_1$},{$\vx_2$},{$\vx_3$}},
    after end axis/.code={
        \node[anchor=south west, font=\bfseries\footnotesize] at (rel axis cs:-0.15,1.02) {a.};
    },
]

\addplot[fill=T-Q-V1, draw=white, area legend] table[
    col sep=comma, x=Index, y=trip,
    restrict expr to domain={\coordindex}{0:2},
] {fig/data-final/exp1-lm/exp1_eigs.csv};

\addplot[fill=T-Q-V2, draw=white, area legend] table[
    col sep=comma, x=Index, y=residual-mode,
    restrict expr to domain={\coordindex}{0:2},
] {fig/data-final/exp1-lm/exp1_eigs.csv};

\addplot[fill=T-Q-MC7, draw=white, area legend] table[
    col sep=comma, x=Index, y=iasc,
    restrict expr to domain={\coordindex}{0:2},
] {fig/data-final/exp1-lm/exp1_eigs.csv};

\addplot[fill=T-Q-V6, draw=white, area legend] table[
    col sep=comma, x=Index, y=timers,
    restrict expr to domain={\coordindex}{0:2},
] {fig/data-final/exp1-lm/exp1_eigs.csv};

\addplot[fill=T-Q-V3, draw=white, area legend] table[
    col sep=comma, x=Index, y=grest-2,
    restrict expr to domain={\coordindex}{0:2},
] {fig/data-final/exp1-lm/exp1_eigs.csv};

\addplot[fill=T-Q-V4, draw=white, area legend] table[
    col sep=comma, x=Index, y=grest-3,
    restrict expr to domain={\coordindex}{0:2},
] {fig/data-final/exp1-lm/exp1_eigs.csv};

\addplot[fill=T-Q-V5, draw=white, area legend] table[
    col sep=comma, x=Index, y=grest-rsvd-100-100,
    restrict expr to domain={\coordindex}{0:2},
] {fig/data-final/exp1-lm/exp1_eigs.csv};

\nextgroupplot[
    title={CM-Collab},
    ybar=0.25pt,
    ymode=log,
    xmin=0.5,
    xmax=3.5,
    log origin=infty,
    bar width=4pt,
    xtick={1,2,3},
    xticklabels={{$\vx_1$},{$\vx_2$},{$\vx_3$}},
]

\addplot[fill=T-Q-V1, draw=white, area legend] table[
    col sep=comma, x=Index, y=trip,
    restrict expr to domain={\coordindex}{3:5},
] {fig/data-final/exp1-lm/exp1_eigs.csv};

\addplot[fill=T-Q-V2, draw=white, area legend] table[
    col sep=comma, x=Index, y=residual-mode,
    restrict expr to domain={\coordindex}{3:5},
] {fig/data-final/exp1-lm/exp1_eigs.csv};

\addplot[fill=T-Q-MC7, draw=white, area legend] table[
    col sep=comma, x=Index, y=iasc,
    restrict expr to domain={\coordindex}{3:5},
] {fig/data-final/exp1-lm/exp1_eigs.csv};

\addplot[fill=T-Q-V6, draw=white, area legend] table[
    col sep=comma, x=Index, y=timers,
    restrict expr to domain={\coordindex}{3:5},
] {fig/data-final/exp1-lm/exp1_eigs.csv};

\addplot[fill=T-Q-V3, draw=white, area legend] table[
    col sep=comma, x=Index, y=grest-2,
    restrict expr to domain={\coordindex}{3:5},
] {fig/data-final/exp1-lm/exp1_eigs.csv};

\addplot[fill=T-Q-V4, draw=white, area legend] table[
    col sep=comma, x=Index, y=grest-3,
    restrict expr to domain={\coordindex}{3:5},
] {fig/data-final/exp1-lm/exp1_eigs.csv};

\addplot[fill=T-Q-V5, draw=white, area legend] table[
    col sep=comma, x=Index, y=grest-rsvd-100-100,
    restrict expr to domain={\coordindex}{3:5},
] {fig/data-final/exp1-lm/exp1_eigs.csv};

\nextgroupplot[
    title={Epinions},
    ybar=0.25pt,
    ymode=log,
    xmin=0.5,
    xmax=3.5,
    log origin=infty,
    bar width=4pt,
    xtick={1,2,3},
    xticklabels={{$\vx_1$},{$\vx_2$},{$\vx_3$}},
    legend to name=leg:exp1-eigs,
]

\addplot[fill=T-Q-V1, draw=white, area legend] table[
    col sep=comma, x=Index, y=trip,
    restrict expr to domain={\coordindex}{6:8},
] {fig/data-final/exp1-lm/exp1_eigs.csv};
\addlegendentry{TRIP}

\addplot[fill=T-Q-V2, draw=white, area legend] table[
    col sep=comma, x=Index, y=residual-mode,
    restrict expr to domain={\coordindex}{6:8},
] {fig/data-final/exp1-lm/exp1_eigs.csv};
\addlegendentry{RM}

\addplot[fill=T-Q-MC7, draw=white, area legend] table[
    col sep=comma, x=Index, y=iasc,
    restrict expr to domain={\coordindex}{6:8},
] {fig/data-final/exp1-lm/exp1_eigs.csv};
\addlegendentry{IASC}

\addplot[fill=T-Q-V6, draw=white, area legend] table[
    col sep=comma, x=Index, y=timers,
    restrict expr to domain={\coordindex}{6:8},
] {fig/data-final/exp1-lm/exp1_eigs.csv};
\addlegendentry{TIMERS}

\addplot[fill=T-Q-V3, draw=white, area legend] table[
    col sep=comma, x=Index, y=grest-2,
    restrict expr to domain={\coordindex}{6:8},
] {fig/data-final/exp1-lm/exp1_eigs.csv};
\addlegendentry{$\proposed_2$}

\addplot[fill=T-Q-V4, draw=white, area legend] table[
    col sep=comma, x=Index, y=grest-3,
    restrict expr to domain={\coordindex}{6:8},
] {fig/data-final/exp1-lm/exp1_eigs.csv};
\addlegendentry{$\proposed_3$}

\addplot[fill=T-Q-V5, draw=white, area legend] table[
    col sep=comma, x=Index, y=grest-rsvd-100-100,
    restrict expr to domain={\coordindex}{6:8},
] {fig/data-final/exp1-lm/exp1_eigs.csv};
\addlegendentry{$\proposed_{\rm RSVD}$}

\nextgroupplot[
    title={Twitch},
    ybar=0.25pt,
    ymode=log,
    xmin=0.5,
    xmax=3.5,
    log origin=infty,
    bar width=4pt,
    xtick={1,2,3},
    xticklabels={{$\vx_1$},{$\vx_2$},{$\vx_3$}},
]

\addplot[fill=T-Q-V1, draw=white, area legend] table[
    col sep=comma, x=Index, y=trip,
    restrict expr to domain={\coordindex}{9:11},
] {fig/data-final/exp1-lm/exp1_eigs.csv};

\addplot[fill=T-Q-V2, draw=white, area legend] table[
    col sep=comma, x=Index, y=residual-mode,
    restrict expr to domain={\coordindex}{9:11},
] {fig/data-final/exp1-lm/exp1_eigs.csv};

\addplot[fill=T-Q-MC7, draw=white, area legend] table[
    col sep=comma, x=Index, y=iasc,
    restrict expr to domain={\coordindex}{9:11},
] {fig/data-final/exp1-lm/exp1_eigs.csv};

\addplot[fill=T-Q-V6, draw=white, area legend] table[
    col sep=comma, x=Index, y=timers,
    restrict expr to domain={\coordindex}{9:11},
] {fig/data-final/exp1-lm/exp1_eigs.csv};

\addplot[fill=T-Q-V3, draw=white, area legend] table[
    col sep=comma, x=Index, y=grest-2,
    restrict expr to domain={\coordindex}{9:11},
] {fig/data-final/exp1-lm/exp1_eigs.csv};

\addplot[fill=T-Q-V4, draw=white, area legend] table[
    col sep=comma, x=Index, y=grest-3,
    restrict expr to domain={\coordindex}{9:11},
] {fig/data-final/exp1-lm/exp1_eigs.csv};

\addplot[fill=T-Q-V5, draw=white, area legend] table[
    col sep=comma, x=Index, y=grest-rsvd-100-100,
    restrict expr to domain={\coordindex}{9:11},
] {fig/data-final/exp1-lm/exp1_eigs.csv};

\end{groupplot}
\end{tikzpicture}%
    
    \hspace{2em} \ref{leg:exp1-eigavg} \\[0.2em]
%
%
%
\begin{tikzpicture}
\begin{groupplot}[%
    group style={
        group size=4 by 1,
        horizontal sep=0.25cm,
        ylabels at=edge left,
        yticklabels at=edge left,
    },
    width=5.4cm,
    xmajorgrids,
    ymajorgrids,
    tick align=outside,
    tick label style={font=\scriptsize},
    tick pos=left,
    ymode=log,
    yticklabel style={xshift=0.5em, yshift=0.2em},
    xlabel={Time},
    ylabel={Angle $\psi$ (radians)},
    label style={font=\scriptsize},
    ylabel style={yshift=-0.5em},
    xlabel style={yshift=0.25em},
    legend style={
        anchor=center,
        at={(0.5, 0.5)},
        font=\scriptsize,
        legend columns=7,
        /tikz/every even column/.append style={
            column sep=1em,
        }
    },
    legend cell align={left},
    title style={yshift=-0.5em, font=\scriptsize},
    enlarge x limits=0.05,
    ymin=0.00005,
    ymax=1,
]

\nextgroupplot[
    title={Crocodile},
    xtick distance=2,
    mark repeat=2,
    mark phase=2,
    xmin=1, xmax=10,
    after end axis/.code={
        \node[anchor=south west, font=\bfseries\footnotesize] at (rel axis cs:-0.15,1.05) {b.};
    },
]

\addplot[color=T-Q-V1, mark=*, MyLinePlot, densely dotted] table[
    col sep=comma, x=TimeStep, y=trip,
    restrict expr to domain={\coordindex}{0:9},
] {fig/data-final/exp1-lm/exp1_subspace_eig.csv};

\addplot[color=T-Q-V2, mark=square*, MyLinePlot, densely dotted] table[
    col sep=comma, x=TimeStep, y=residual-mode,
    restrict expr to domain={\coordindex}{0:9},
] {fig/data-final/exp1-lm/exp1_subspace_eig.csv};

\addplot[color=T-Q-MC7, mark=triangle*, MyLinePlot, densely dotted] table[
    col sep=comma, x=TimeStep, y=iasc,
    restrict expr to domain={\coordindex}{0:9},
] {fig/data-final/exp1-lm/exp1_subspace_eig.csv};

\addplot[color=T-Q-V6, mark=diamond*, MyLinePlot, densely dotted, only marks, mark repeat=1, mark phase=0] table[
    col sep=comma, x=TimeStep, y=timers,
    restrict expr to domain={\coordindex}{0:9},
] {fig/data-final/exp1-lm/exp1_subspace_eig.csv};

\addplot[color=T-Q-V3, mark=*, MyLinePlot] table[
    col sep=comma, x=TimeStep, y=grest-2,
    restrict expr to domain={\coordindex}{0:9},
] {fig/data-final/exp1-lm/exp1_subspace_eig.csv};

\addplot[color=T-Q-V4, mark=*, MyLinePlot] table[
    col sep=comma, x=TimeStep, y=grest-3,
    restrict expr to domain={\coordindex}{0:9},
] {fig/data-final/exp1-lm/exp1_subspace_eig.csv};

\addplot[color=T-Q-V5, mark=triangle*, MyLinePlot] table[
    col sep=comma, x=TimeStep, y=grest-rsvd-100-100,
    restrict expr to domain={\coordindex}{0:9},
] {fig/data-final/exp1-lm/exp1_subspace_eig.csv};

\nextgroupplot[
    title={CM-Collab},
    xtick distance=4,
    mark repeat=4,
    mark phase=4,
    xmin=1, xmax=20,
]

\addplot[color=T-Q-V1, mark=*, MyLinePlot, densely dotted] table[
    col sep=comma, x=TimeStep, y=trip,
    restrict expr to domain={\coordindex}{10:29},
] {fig/data-final/exp1-lm/exp1_subspace_eig.csv};

\addplot[color=T-Q-V2, mark=square*, MyLinePlot, densely dotted] table[
    col sep=comma, x=TimeStep, y=residual-mode,
    restrict expr to domain={\coordindex}{10:29},
] {fig/data-final/exp1-lm/exp1_subspace_eig.csv};

\addplot[color=T-Q-MC7, mark=triangle*, MyLinePlot, densely dotted] table[
    col sep=comma, x=TimeStep, y=iasc,
    restrict expr to domain={\coordindex}{10:29},
] {fig/data-final/exp1-lm/exp1_subspace_eig.csv};

\addplot[color=T-Q-V6, mark=diamond*, MyLinePlot, densely dotted, only marks, mark repeat=1, mark phase=0] table[
    col sep=comma, x=TimeStep, y=timers,
    restrict expr to domain={\coordindex}{10:29},
] {fig/data-final/exp1-lm/exp1_subspace_eig.csv};

\addplot[color=T-Q-V3, mark=*, MyLinePlot] table[
    col sep=comma, x=TimeStep, y=grest-2,
    restrict expr to domain={\coordindex}{10:29},
] {fig/data-final/exp1-lm/exp1_subspace_eig.csv};

\addplot[color=T-Q-V4, mark=*, MyLinePlot] table[
    col sep=comma, x=TimeStep, y=grest-3,
    restrict expr to domain={\coordindex}{10:29},
] {fig/data-final/exp1-lm/exp1_subspace_eig.csv};

\addplot[color=T-Q-V5, mark=triangle*, MyLinePlot] table[
    col sep=comma, x=TimeStep, y=grest-rsvd-100-100,
    restrict expr to domain={\coordindex}{10:29},
] {fig/data-final/exp1-lm/exp1_subspace_eig.csv};

\nextgroupplot[
    title={Epinions},
    xtick distance=10,
    mark repeat=10,
    mark phase=10,
    xmin=1, xmax=50,
    legend to name=leg:exp1-eigavg,
]

\addplot[color=T-Q-V1, mark=*, MyLinePlot, densely dotted] table[
    col sep=comma, x=TimeStep, y=trip,
    restrict expr to domain={\coordindex}{30:79},
] {fig/data-final/exp1-lm/exp1_subspace_eig.csv};
\addlegendentry{TRIP}

\addplot[color=T-Q-V2, mark=square*, MyLinePlot, densely dotted] table[
    col sep=comma, x=TimeStep, y=residual-mode,
    restrict expr to domain={\coordindex}{30:79},
] {fig/data-final/exp1-lm/exp1_subspace_eig.csv};
\addlegendentry{RM}

\addplot[color=T-Q-MC7, mark=triangle*, MyLinePlot, densely dotted] table[
    col sep=comma, x=TimeStep, y=iasc,
    restrict expr to domain={\coordindex}{30:79},
] {fig/data-final/exp1-lm/exp1_subspace_eig.csv};
\addlegendentry{IASC}

\addplot[color=T-Q-V6, mark=diamond*, MyLinePlot, densely dotted, only marks, mark repeat=2, mark phase=0] table[
    col sep=comma, x=TimeStep, y=timers,
    restrict expr to domain={\coordindex}{30:79},
] {fig/data-final/exp1-lm/exp1_subspace_eig.csv};
\addlegendentry{TIMERS}

\addplot[color=T-Q-V3, mark=*, MyLinePlot] table[
    col sep=comma, x=TimeStep, y=grest-2,
    restrict expr to domain={\coordindex}{30:79},
] {fig/data-final/exp1-lm/exp1_subspace_eig.csv};
\addlegendentry{$\proposed_{2}$}

\addplot[color=T-Q-V4, mark=*, MyLinePlot] table[
    col sep=comma, x=TimeStep, y=grest-3,
    restrict expr to domain={\coordindex}{30:79},
] {fig/data-final/exp1-lm/exp1_subspace_eig.csv};
\addlegendentry{$\proposed_{3}$}

\addplot[color=T-Q-V5, mark=triangle*, MyLinePlot] table[
    col sep=comma, x=TimeStep, y=grest-rsvd-100-100,
    restrict expr to domain={\coordindex}{30:79},
] {fig/data-final/exp1-lm/exp1_subspace_eig.csv};
\addlegendentry{$\proposed_{\rm RSVD}$}

\nextgroupplot[
    title={Twitch},
    xtick distance=20,
    mark repeat=20,
    mark phase=20,
    xmin=1, xmax=100,
]

\addplot[color=T-Q-V1, mark=*, MyLinePlot, densely dotted] table[
    col sep=comma, x=TimeStep, y=trip,
    restrict expr to domain={\coordindex}{80:179},
] {fig/data-final/exp1-lm/exp1_subspace_eig.csv};

\addplot[color=T-Q-V2, mark=square*, MyLinePlot, densely dotted] table[
    col sep=comma, x=TimeStep, y=residual-mode,
    restrict expr to domain={\coordindex}{80:179},
] {fig/data-final/exp1-lm/exp1_subspace_eig.csv};

\addplot[color=T-Q-MC7, mark=triangle*, MyLinePlot, densely dotted] table[
    col sep=comma, x=TimeStep, y=iasc,
    restrict expr to domain={\coordindex}{80:179},
] {fig/data-final/exp1-lm/exp1_subspace_eig.csv};

\addplot[color=T-Q-V3, mark=*, MyLinePlot] table[
    col sep=comma, x=TimeStep, y=grest-2,
    restrict expr to domain={\coordindex}{80:179},
] {fig/data-final/exp1-lm/exp1_subspace_eig.csv};

\addplot[color=T-Q-V4, mark=*, MyLinePlot] table[
    col sep=comma, x=TimeStep, y=grest-3,
    restrict expr to domain={\coordindex}{80:179},
] {fig/data-final/exp1-lm/exp1_subspace_eig.csv};

\addplot[color=T-Q-V5, mark=triangle*, MyLinePlot] table[
    col sep=comma, x=TimeStep, y=grest-rsvd-100-100,
    restrict expr to domain={\coordindex}{80:179},
] {fig/data-final/exp1-lm/exp1_subspace_eig.csv};

\end{groupplot}
\end{tikzpicture}%
    \vspace{-1em}
    \caption{Eigenvector approximation results for dynamic graphs constructed from static graphs (Scenario 1). (a). shows the average of $\psi_{i,t}$ over time for the first three leading eigenvectors. (b). shows the average of $\psi_{i,t}$ over the first $32$ leading eigenvectors as a function of $t$. Results for TIMERS are not reported for the Twitch dataset due to its high time requirement.}
    \label{fig:eig-approx-case1}
\end{figure*}
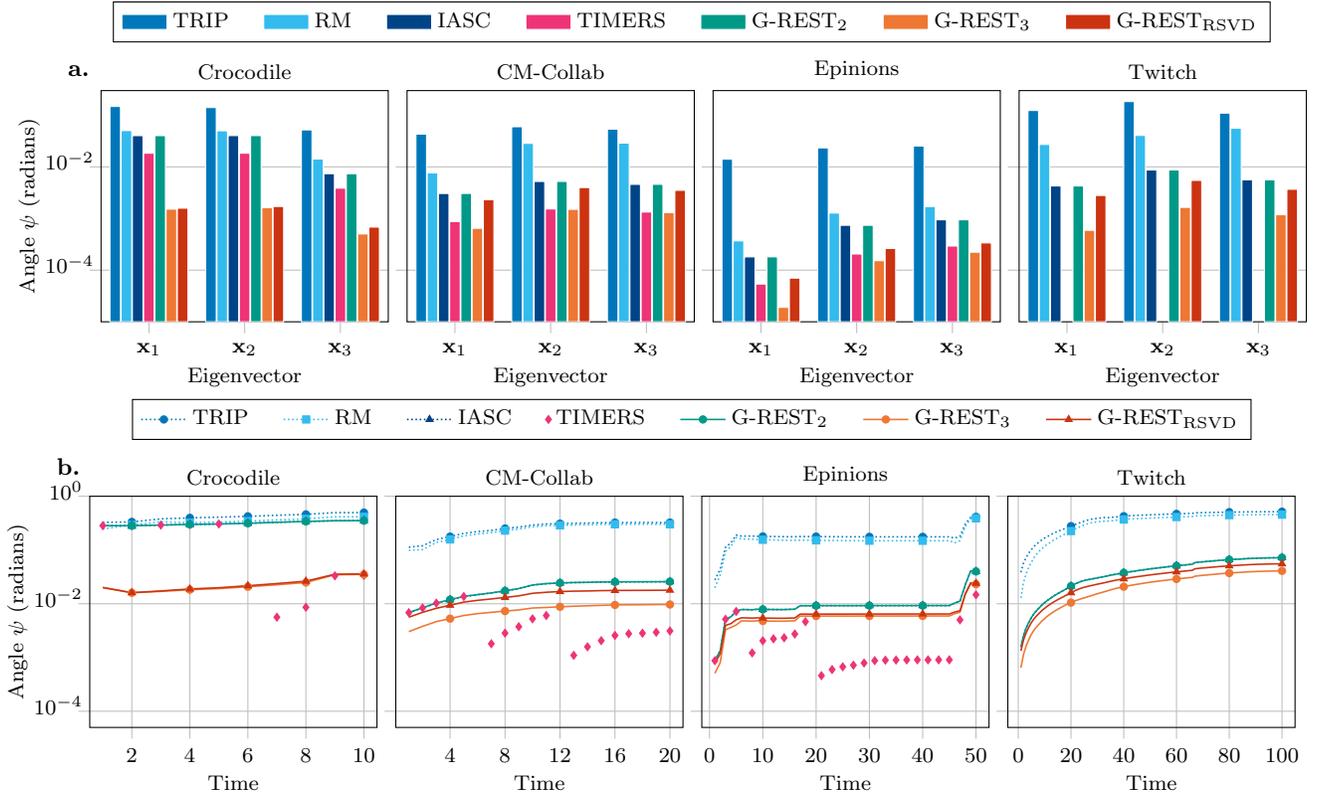

The first set of numerical experiments evaluates the performance of the compared algorithms in terms of eigenvector estimation for various dynamic graphs. For a dynamic graph $\{\calG^{(t)}\}_{t=0}^T$, let $\vx_{i}^{(t)}$ be the reference $i$th leading eigenvector of $\mA^{(t)}$ computed via MATLAB's truncated eigendecomposition function \verb|eigs| and let $\tilde{\vx}_{i}^{(t)}$ be its estimate returned by one of the  tracking algorithms for $t>0$. The figure of merit is the angle between $\vx_{i}^{(t)}$ and $\tilde{\vx}_{i}^{(t)}$ measured as:
\begin{equation}
    \psi_{i}^{(t)} = \arccos \left( 
        | {\vx_{i}^{(t)}}^{\top} \widetilde{\vx}_{i}^{(t)}| 
    \right) \in [0,\pi],
\end{equation}
where both $\vx_{i}^{(t)}$ and $\tilde{\vx}_{i}^{(t)}$ are normalized to have unit norm. Smaller values of $\psi_{i}^{(t)}$ indicate that $\vx_{i}^{(t)}$ and its estimate $\widetilde{\vx}_{i}^{(t)}$ are close. In the following, the number of tracked eigenpairs is set to $K=64$ for all algorithms and datasets.

\begin{figure*}[t]
    \centering
    \hspace{2em} \ref{leg:exp2-eigs} \\[0.2em]
%
%
%
\begin{tikzpicture}
\begin{groupplot}[%
    group style={
        group size=4 by 1,
        horizontal sep=0.25cm,
        ylabels at=edge left,
        yticklabels at=edge left,
    },
    width=5.4cm,
    ymajorgrids,
    tick align=outside,
    tick label style={font=\scriptsize},
    tick pos=left,
    yticklabel style={xshift=0.5em, yshift=0.2em},
    xlabel={Eigenvector},
    ylabel={Angle $\psi$ (radians)},
    label style={font=\scriptsize},
    ylabel style={yshift=-0.5em},
    xlabel style={yshift=0.25em},
    legend style={
        anchor=center,
        at={(0.5, 0.5)},
        font=\scriptsize,
        legend columns=7,
        /tikz/every even column/.append style={
            column sep=1em,
            draw=white,
        },
    },
    legend cell align={left},
    title style={yshift=-0.5em, font=\scriptsize},
    ymin=0.005,
    ymax=1
]

\nextgroupplot[
    title={MathOverflow},
    ybar=0.25pt,
    ymode=log,
    xmin=0.5,
    xmax=3.5,
    log origin=infty,
    bar width=4pt,
    xtick={1,2,3},
    xticklabels={{$\vx_1$},{$\vx_2$},{$\vx_3$}},
    after end axis/.code={
        \node[anchor=south west, font=\bfseries\footnotesize] at (rel axis cs:-0.15,1.02) {a.};
    },
]

\addplot[fill=T-Q-V1, draw=white, area legend] table[
    col sep=comma, x=Index, y=trip,
    restrict expr to domain={\coordindex}{0:2},
] {fig/data-final/exp2-lm/exp2_eigs.csv};

\addplot[fill=T-Q-V2, draw=white, area legend] table[
    col sep=comma, x=Index, y=residual-mode,
    restrict expr to domain={\coordindex}{0:2},
] {fig/data-final/exp2-lm/exp2_eigs.csv};

\addplot[fill=T-Q-MC7, draw=white, area legend] table[
    col sep=comma, x=Index, y=iasc,
    restrict expr to domain={\coordindex}{0:2},
] {fig/data-final/exp2-lm/exp2_eigs.csv};

\addplot[fill=T-Q-V6, draw=white, area legend] table[
    col sep=comma, x=Index, y=timers,
    restrict expr to domain={\coordindex}{0:2},
] {fig/data-final/exp2-lm/exp2_eigs.csv};

\addplot[fill=T-Q-V3, draw=white, area legend] table[
    col sep=comma, x=Index, y=grest-2,
    restrict expr to domain={\coordindex}{0:2},
] {fig/data-final/exp2-lm/exp2_eigs.csv};

\addplot[fill=T-Q-V4, draw=white, area legend] table[
    col sep=comma, x=Index, y=grest-3,
    restrict expr to domain={\coordindex}{0:2},
] {fig/data-final/exp2-lm/exp2_eigs.csv};

\addplot[fill=T-Q-V5, draw=white, area legend] table[
    col sep=comma, x=Index, y=grest-rsvd-100-100,
    restrict expr to domain={\coordindex}{0:2},
] {fig/data-final/exp2-lm/exp2_eigs.csv};

\nextgroupplot[
    title={Tech},
    ybar=0.25pt,
    ymode=log,
    xmin=0.5,
    xmax=3.5,
    log origin=infty,
    bar width=4pt,
    xtick={1,2,3},
    xticklabels={{$\vx_1$},{$\vx_2$},{$\vx_3$}},
]

\addplot[fill=T-Q-V1, draw=white, area legend] table[
    col sep=comma, x=Index, y=trip,
    restrict expr to domain={\coordindex}{3:5},
] {fig/data-final/exp2-lm/exp2_eigs.csv};

\addplot[fill=T-Q-V2, draw=white, area legend] table[
    col sep=comma, x=Index, y=residual-mode,
    restrict expr to domain={\coordindex}{3:5},
] {fig/data-final/exp2-lm/exp2_eigs.csv};

\addplot[fill=T-Q-MC7, draw=white, area legend] table[
    col sep=comma, x=Index, y=iasc,
    restrict expr to domain={\coordindex}{3:5},
] {fig/data-final/exp2-lm/exp2_eigs.csv};

\addplot[fill=T-Q-V6, draw=white, area legend] table[
    col sep=comma, x=Index, y=timers,
    restrict expr to domain={\coordindex}{3:5},
] {fig/data-final/exp2-lm/exp2_eigs.csv};

\addplot[fill=T-Q-V3, draw=white, area legend] table[
    col sep=comma, x=Index, y=grest-2,
    restrict expr to domain={\coordindex}{3:5},
] {fig/data-final/exp2-lm/exp2_eigs.csv};

\addplot[fill=T-Q-V4, draw=white, area legend] table[
    col sep=comma, x=Index, y=grest-3,
    restrict expr to domain={\coordindex}{3:5},
] {fig/data-final/exp2-lm/exp2_eigs.csv};

\addplot[fill=T-Q-V5, draw=white, area legend] table[
    col sep=comma, x=Index, y=grest-rsvd-100-100,
    restrict expr to domain={\coordindex}{3:5},
] {fig/data-final/exp2-lm/exp2_eigs.csv};

\nextgroupplot[
    title={Enron},
    ybar=0.25pt,
    ymode=log,
    xmin=0.5,
    xmax=3.5,
    log origin=infty,
    bar width=4pt,
    xtick={1,2,3},
    xticklabels={{$\vx_1$},{$\vx_2$},{$\vx_3$}},
]

\addplot[fill=T-Q-V1, draw=white, area legend] table[
    col sep=comma, x=Index, y=trip,
    restrict expr to domain={\coordindex}{6:8},
] {fig/data-final/exp2-lm/exp2_eigs.csv};

\addplot[fill=T-Q-V2, draw=white, area legend] table[
    col sep=comma, x=Index, y=residual-mode,
    restrict expr to domain={\coordindex}{6:8},
] {fig/data-final/exp2-lm/exp2_eigs.csv};

\addplot[fill=T-Q-MC7, draw=white, area legend] table[
    col sep=comma, x=Index, y=iasc,
    restrict expr to domain={\coordindex}{6:8},
] {fig/data-final/exp2-lm/exp2_eigs.csv};

\addplot[fill=T-Q-V6, draw=white, area legend] table[
    col sep=comma, x=Index, y=timers,
    restrict expr to domain={\coordindex}{6:8},
] {fig/data-final/exp2-lm/exp2_eigs.csv};

\addplot[fill=T-Q-V3, draw=white, area legend] table[
    col sep=comma, x=Index, y=grest-2,
    restrict expr to domain={\coordindex}{6:8},
] {fig/data-final/exp2-lm/exp2_eigs.csv};

\addplot[fill=T-Q-V4, draw=white, area legend] table[
    col sep=comma, x=Index, y=grest-3,
    restrict expr to domain={\coordindex}{6:8},
] {fig/data-final/exp2-lm/exp2_eigs.csv};

\addplot[fill=T-Q-V5, draw=white, area legend] table[
    col sep=comma, x=Index, y=grest-rsvd-100-100,
    restrict expr to domain={\coordindex}{6:8},
] {fig/data-final/exp2-lm/exp2_eigs.csv};

\nextgroupplot[
    title={AskUbuntu},
    ybar=0.25pt,
    ymode=log,
    xmin=0.5,
    xmax=3.5,
    log origin=infty,
    bar width=4pt,
    xtick={1,2,3},
    xticklabels={{$\vx_1$},{$\vx_2$},{$\vx_3$}},
    legend to name=leg:exp2-eigs,
]

\addplot[fill=T-Q-V1, draw=white, area legend] table[
    col sep=comma, x=Index, y=trip,
    restrict expr to domain={\coordindex}{9:11},
] {fig/data-final/exp2-lm/exp2_eigs.csv};
\addlegendentry{TRIP}

\addplot[fill=T-Q-V2, draw=white, area legend] table[
    col sep=comma, x=Index, y=residual-mode,
    restrict expr to domain={\coordindex}{9:11},
] {fig/data-final/exp2-lm/exp2_eigs.csv};
\addlegendentry{RM}

\addplot[fill=T-Q-MC7, draw=white, area legend] table[
    col sep=comma, x=Index, y=iasc,
    restrict expr to domain={\coordindex}{9:11},
] {fig/data-final/exp2-lm/exp2_eigs.csv};
\addlegendentry{IASC}

\addplot[fill=T-Q-V6, draw=white, area legend] table[
    col sep=comma, x=Index, y=timers,
    restrict expr to domain={\coordindex}{9:11},
] {fig/data-final/exp2-lm/exp2_eigs.csv};
\addlegendentry{TIMERS}

\addplot[fill=T-Q-V3, draw=white, area legend] table[
    col sep=comma, x=Index, y=grest-2,
    restrict expr to domain={\coordindex}{9:11},
] {fig/data-final/exp2-lm/exp2_eigs.csv};
\addlegendentry{$\proposed_2$}

\addplot[fill=T-Q-V4, draw=white, area legend] table[
    col sep=comma, x=Index, y=grest-3,
    restrict expr to domain={\coordindex}{9:11},
] {fig/data-final/exp2-lm/exp2_eigs.csv};
\addlegendentry{$\proposed_3$}

\addplot[fill=T-Q-V5, draw=white, area legend] table[
    col sep=comma, x=Index, y=grest-rsvd-100-100,
    restrict expr to domain={\coordindex}{9:11},
] {fig/data-final/exp2-lm/exp2_eigs.csv};
\addlegendentry{$\proposed_{\rm RSVD}$}

\end{groupplot}
\end{tikzpicture}%
    
    \hspace{2em} \ref{leg:exp2-eigavg} \\[0.2em]
    \input{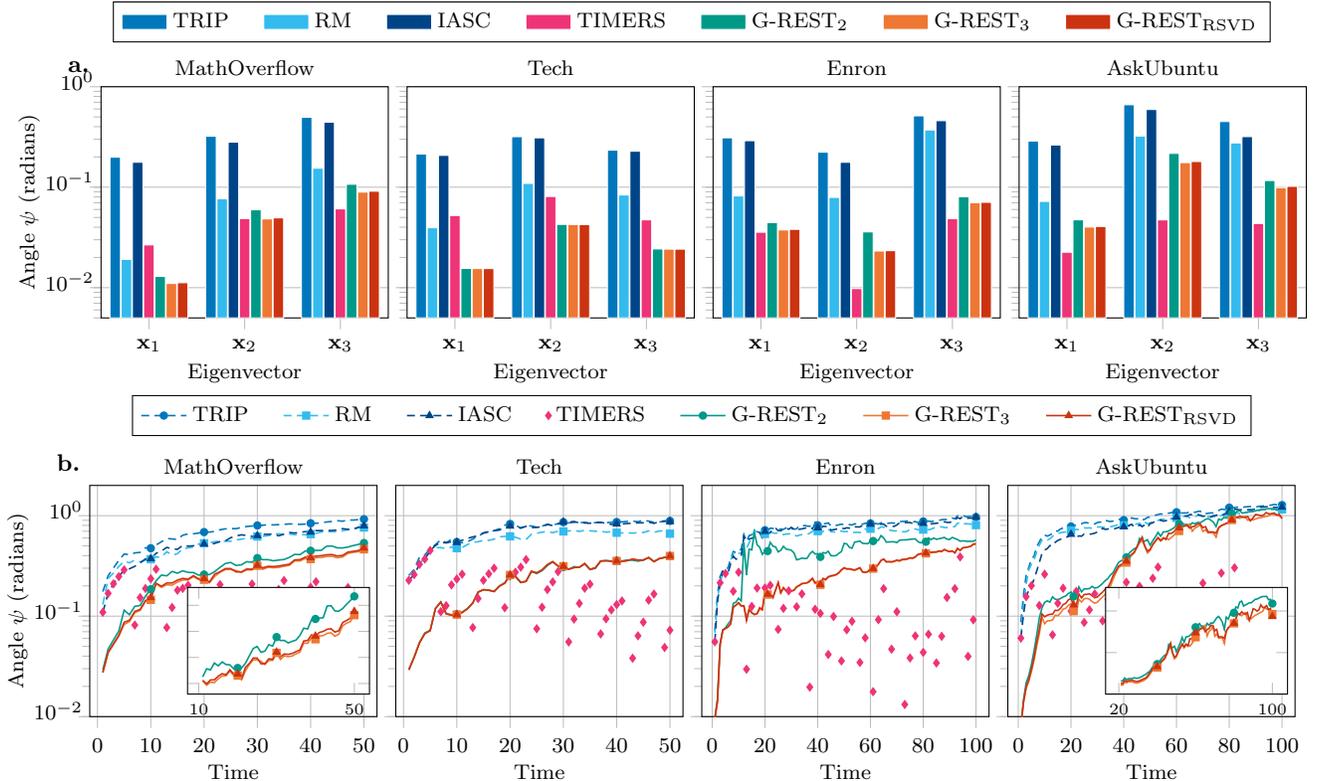}
    \vspace{-1em}
    \caption{Eigenvector approximation results for graphs with timestampted edges (Scenario 2). (a). shows the average of $\psi_{i,t}$ over time for the first three leading eigenvectors. (b). shows the average of $\psi_{i,t}$ over the first $32$ leading eigenvectors as a function of $t$. For some datasets in (b), insets zoom into the performance of the proposed methods to showcase their differences.}
    \label{fig:eig-approx-real-dyn-graphs}
\end{figure*}

\textbf{Scenario 1 - Creating dynamic graphs from static ones:} Here, all algorithms are evaluated using dynamic graphs constructed from static graph datasets, where edges are not timestamped. These datasets allow us to control the characteristics of dynamic graphs, such the size of graph expansion or time points. Static graphs (denoted as \emph{Type S}) are obtained from the Stanford Network Analysis Project (SNAP~\cite{snapnets}) and their properties are reported in Table \ref{tab:dataset_details}. All datasets are undirected graphs with the exception of Epinions, which is converted to an undirected one by ignoring edge directionality. For a dataset corresponding to a static graph $\calG = (\calV, \calE)$ with $|\calV| = N$, a dynamic graph $\{\calG^{(t)}\}_{t=0}^{T}$ is constructed as follows: First $N^{(0)}$ nodes with the highest degrees in $\calG$ are selected to form the initial vertex set $\calV^{(0)}$, whose induced subgraph is denoted as $\calG^{(0)} = (\calV^{(0)}, \calE^{(0)})$. The remaining graphs $\calG^{(t)}=(\calV^{(t)}, \calE^{(t)})$ for $t > 0$ are determined using graph expansion, i.e., $\calG^{(t)}$ is the subgraph induced by $\calV^{(t)} = \calV^{(t-1)} \cup \calC^{(t)}$ where $\calC^{(t)}$ is the set of $S^{(t)}$ nodes in $\calV \setminus \calV^{(t-1)}$ with the highest degrees in the original graph $\calG$. In the following, $N^{(0)} = \lfloor N/2 \rfloor $, $S^{(t)} = \lfloor (N-N^{(0)})/T \rfloor$. Different values of $T$ are used for each dataset to evaluate algorithm performance across multiple time horizons.

Fig. \ref{fig:eig-approx-case1} plots the approximation accuracy of the considered algorithms for four datasets. Fig. \ref{fig:eig-approx-case1}(a) shows $\psi_{i}^{(t)}$ averaged over time for the leading three eigenvectors, i.e., $\frac{1}{T}\sum_{t=1}^{T} \psi_{i}^{(t)}$ for $1 \leq i \leq 3$; while Fig. \ref{fig:eig-approx-case1}(b) reports the average of $\psi_{i}^{(t)}$ over the leading $32$ eigenvectors, i.e., $\frac{1}{32}\sum_{i=1}^{32} \psi_{i}^{(t)}$,  as a function of time. TIMERS results are not reported for the Twitch dataset, as it did not provide results within a reasonable time frame. $\proposed_2$, which uses the same subspace as RM, outperforms RM, indicating that searching for the optimal linear combination within the subspace through RR projection leads to improved results compared to using fixed coefficients. Both RM and $\proposed_2$ outperform TRIP, as the latter lacks information from untracked eigenvectors. IASC, which also uses RR, exhibits the same performance as $\proposed_2$, which provides another indication for the effectiveness of RR for tracking eigenvectors of dynamic graphs. $\proposed_3$ and $\proposed_{\rm RSVD}$ provide significant improvements in eigenvector approximation accuracy, as their projection space includes information from new nodes. $\proposed_3$ outperforms $\proposed_{\rm RSVD}$ since the latter approximates the projection subspace using RSVD. Still, $\proposed_{\rm RSVD}$ provides reasonable performance while enjoying lower computational complexity complexity, as will be discussed later. Finally, since TIMERS can trigger a truncated eigendecomposition of $\mA^{(t)}$, it improves upon IASC and enjoys the best approximation accuracy across time, as can be seen in Fig. \ref{fig:eig-approx-case1}(b). Note, however, that TIMERS has a significantly larger memory footprint than $\proposed$. Interestingly, the results of Fig.~\ref{fig:eig-approx-case1}(a) indicate that  $\proposed_3$ provides better accuracy than TIMERS in estimating first three eigenvectors.

\textbf{Scenario 2 - Graphs with timestamped edges:} This scenario considers real graph datasets with timestamped edges. As with the static graphs, datasets are obtained from SNAP with the exception of Tech dataset, which is acquired from the Network Data Repository \cite{rossi2015network}. Datasets and their properties are listed in Table \ref{tab:dataset_details} (denoted as \emph{Type D}). From $M$ timestamped edges, ordered based on their time of appearance, a dynamic graph $\{\calG^{(t)}\}_{t=0}^{T}$ is constructed as follows: $\calG^{(0)}=(\calV^{(0)}, \calE^{(0)})$ is constructed by setting $\calE^{(0)}$ to the first $M^{(0)}$ timestamped edges with $\calV^{(0)}$ being the set nodes involved in $\calE^{(0)}$. The remaining $\calG^{(t)} = (\calV^{(t)}, \calE^{(t)})$'s are determined by expanding the edge set by $M^{(t)}$ edges, that is, $\calE^{(t)} = \calE^{(t-1)} \cup \widehat{\calE}^{(t)}$ where $\widehat{\calE}^{(t)}$ includes the next $M^{(t)}$ timestamped edges. These graphs contain both topological updates and graph expansion.

Results for these datasets are reported in Fig. \ref{fig:eig-approx-real-dyn-graphs} where $M^{(0)} = \lfloor M/2 \rfloor$, $M^{(t)} = \lfloor (M - M^{(0)})/T\rfloor$, with $T=50$ for the MathOverflow and Tech datasets, and $T=100$ for the remaining two datasets. As in the previous scenario, the performance of estimating the leading three eigenvectors and the leading $32$ eigenvectors is shown in Figs. \ref{fig:eig-approx-real-dyn-graphs}(a) and \ref{fig:eig-approx-real-dyn-graphs}(b), respectively. Overall, TIMERS followed by $\proposed_3$ and $\proposed_{\rm RSVD}$ provide the best estimation performance. Consistent with the first scenario, $\proposed_2$'s accuracy is better than RM and TRIP, while IACS is not performing as well, possibly due to its inability to capture the topological updates that occur in the dynamic graphs. %
\subsection{Runtime and Computational Efficiency}
\label{ssec:runtimes}
Next, the compared algorithms are evaluated in terms of runtime, which serves as a proxy for computational complexity. In particular, all algorithms are applied to the datasets considered in Scenario 1 and Scenario 2 of the previous subsection to measure the time required by each one to track the eigendecomposition of graph adjacency matrices. Results are reported in Fig. \ref{fig:run-times}. The run time of \verb|eigs| is also reported as baseline. Results for the datasets of Scenario 1 are shown in Fig. \ref{fig:run-times}(a). Both $\proposed_3$ and TIMERS exhibit comparable runtimes to \verb|eigs|. TIMERS’ high time complexity is due to the computation of a proxy for the eigenvector approximation error, which occurs at each time step. The number of nodes added at each time-step for these datasets is between 500 and 1000, resulting in high runtimes for $\proposed_3$. $\proposed_{\rm RSVD}$ exhibits significantly lower computational complexity. This suggests that randomized SVD is an effective approach to reduce time complexity of $\proposed_3$, when graphs are expanded with large numbers of nodes. At the same time $\proposed_{\rm RSVD}$ enjoys comparable eigenpair approximation accuracy with $\proposed_3$, as discussed in previous section. Among the remaining algorithms, TRIP followed by RM are the fastest. $\proposed_2$ requires more time than TRIP and RM, while being faster than IASC. 

Results for the datasets of Scenario 2 are reported in Fig. \ref{fig:run-times}(b). TRIP, RM, IASC and $\proposed_2$ exhibit similar behavior to scenario 1. TIMERS has a runtime comparable to \verb|eigs|. $\proposed_3$ requires less time than \verb|eigs| for MathOverflow and Tech datasets, while having similar runtime for the remaining two datasets. For the former datasets, the average number of new nodes arriving at each time step is $270$ and $42$, respectively, which translates to lower runtimes for $\proposed_3$. For the latter two datasets the size of graph expansion at each time is greater than $400$, thus $\proposed_3$ exhibits increased runtimes, similar to Scenario 1. Nevertheless, as before, the RSVD version of $\proposed_3$ reduces the time complexity while providing good estimation accuracy. The results from this and the previous subsection imply that $\proposed_{\rm RSVD}$ provides the best performance while enjoying reasonable computational complexity. $\proposed_3$ is viable when the size of graph expansion is small as observed in MathOverflow and Tech datasets; for large graph expansions $\proposed_{\rm RSVD}$ is suggested. 

To sum up the results of Secs.~\ref{ssec:eigs-est-results} and~\ref{ssec:runtimes}, the proposed approach can achieve high eigenvector approximation accuracy, while maintaining reasonable computational and memory complexity.

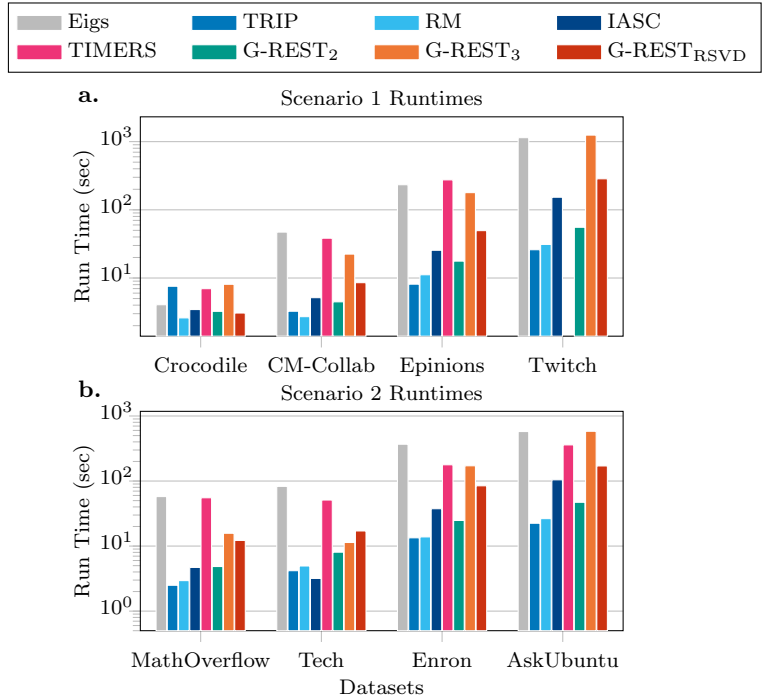
\begin{figure}[t]
    \centering
    \hspace{2em} \ref{leg:run-time} \\[0.2em]
%
%
%
\begin{tikzpicture}
\begin{groupplot}[%
    group style={
        group size=1 by 2,
        vertical sep=1cm,
        xlabels at=edge bottom,
        yticklabels at=edge left,
    },
    width=8cm,
    height=4.5cm,
    ymajorgrids,
    tick align=outside,
    tick label style={font=\scriptsize},
    tick pos=left,
    ymode=log,
    yticklabel style={xshift=0.5em, yshift=0.2em},
    xlabel={Datasets},
    ylabel={Run Time (sec)},
    label style={font=\scriptsize},
    ylabel style={yshift=-0.5em},
    xlabel style={yshift=0.25em},
    legend style={
        anchor=center,
        at={(0.5, 0.5)},
        font=\scriptsize,
        legend columns=4,
        /tikz/every even column/.append style={
            column sep=1em,
        }
    },
    legend cell align={left},
    title style={yshift=-0.5em, font=\scriptsize},
]

\nextgroupplot[
    title={Scenario 1 Runtimes},
    ybar=0.25pt,
    ymode=log,
    xmin=0.5,
    xmax=4.5,
    log origin=infty,
    bar width=4pt,
    xtick={1,2,3,4},
    xticklabels={{Crocodile},{CM-Collab},{Epinions},{Twitch}},
    after end axis/.code={
        \node[anchor=south west, font=\bfseries\footnotesize] at (rel axis cs:-0.15,1.02) {a.};
    },
]

\addplot[fill=T-Q-V0, draw=white, area legend] table[
    col sep=comma, y=eigs, x=Index
] {fig/data-final/exp1-lm/exp1_run_times.csv};

\addplot[fill=T-Q-V1, draw=white, area legend] table[
    col sep=comma, y=trip, x=Index
] {fig/data-final/exp1-lm/exp1_run_times.csv};

\addplot[fill=T-Q-V2, draw=white, area legend] table[
    col sep=comma, y=residual-mode, x=Index
] {fig/data-final/exp1-lm/exp1_run_times.csv};

\addplot[fill=T-Q-MC7, draw=white, area legend] table[
    col sep=comma, y=iasc, x=Index,
] {fig/data-final/exp1-lm/exp1_run_times.csv};

\addplot[fill=T-Q-V6, draw=white, area legend] table[
    col sep=comma, y=timers, x=Index,
] {fig/data-final/exp1-lm/exp1_run_times.csv};

\addplot[fill=T-Q-V3, draw=white, area legend] table[
    col sep=comma, y=grest-2, x=Index
] {fig/data-final/exp1-lm/exp1_run_times.csv};

\addplot[fill=T-Q-V4, draw=white, area legend] table[
    col sep=comma, y=grest-3, x=Index
] {fig/data-final/exp1-lm/exp1_run_times.csv};

\addplot[fill=T-Q-V5, draw=white, area legend] table[
    col sep=comma, y=grest-rsvd-100-100, x=Index
] {fig/data-final/exp1-lm/exp1_run_times.csv};

\nextgroupplot[
    title={Scenario 2 Runtimes},
    ybar=0.25pt,
    ymode=log,
    xmin=0.5,
    xmax=4.5,
    log origin=infty,
    bar width=4pt,
    xtick={1,2,3,4},
    xticklabels={{MathOverflow},{Tech},{Enron},{AskUbuntu}},
    after end axis/.code={
        \node[anchor=south west, font=\bfseries\footnotesize] at (rel axis cs:-0.15,1.02) {b.};
    },
    legend to name=leg:run-time,
    ymin=0.5
]

\addplot[fill=T-Q-V0, draw=white, area legend] table[
    col sep=comma, y=eigs, x=Index
] {fig/data-final/exp2-lm/exp2_run_times.csv};
\addlegendentry{Eigs}

\addplot[fill=T-Q-V1, draw=white, area legend] table[
    col sep=comma, y=trip, x=Index
] {fig/data-final/exp2-lm/exp2_run_times.csv};
\addlegendentry{TRIP}

\addplot[fill=T-Q-V2, draw=white, area legend] table[
    col sep=comma, y=residual-mode, x=Index
] {fig/data-final/exp2-lm/exp2_run_times.csv};
\addlegendentry{RM}

\addplot[fill=T-Q-MC7, draw=white, area legend] table[
    col sep=comma, y=iasc, x=Index,
] {fig/data-final/exp2-lm/exp2_run_times.csv};
\addlegendentry{IASC}

\addplot[fill=T-Q-V6, draw=white, area legend] table[
    col sep=comma, y=timers, x=Index,
] {fig/data-final/exp2-lm/exp2_run_times.csv};
\addlegendentry{TIMERS}

\addplot[fill=T-Q-V3, draw=white, area legend] table[
    col sep=comma, y=grest-2, x=Index
] {fig/data-final/exp2-lm/exp2_run_times.csv};
\addlegendentry{$\proposed_{2}$}

\addplot[fill=T-Q-V4, draw=white, area legend] table[
    col sep=comma, y=grest-3, x=Index
] {fig/data-final/exp2-lm/exp2_run_times.csv};
\addlegendentry{$\proposed_{3}$}

\addplot[fill=T-Q-V5, draw=white, area legend] table[
    col sep=comma, y=grest-rsvd-100-100, x=Index
] {fig/data-final/exp2-lm/exp2_run_times.csv};
\addlegendentry{$\proposed_{\rm RSVD}$}

\end{groupplot}
\end{tikzpicture}%
    \vspace{-1em}
    \caption{Run times of different algorithms in seconds for the scenaria of Section \ref{ssec:eigs-est-results}. (a). reports the times needed to process datasets of Scenario 1, while (b). reports the times for Scenario 2.}
    \label{fig:run-times}
\end{figure}

\subsection{Complexity and Accuracy Trade-off for RSVD}

The results of the previous subsections indicate that RSVD can reduce the time complexity of $\proposed_3$, while providing comparable eigenvector approximation accuracy. Here, the effectiveness of RSVD is further investigated, on the CM-Collab dataset, by studying how its time complexity and performance varies as the oversampling parameter $P$ and the rank parameter $L$ change (See Sec.~\ref{ssec:RSVD}). For each pair $(P,L)$, run time and eigenvector estimation accuracy are measured. As with the results of Sec.~\ref{ssec:eigs-est-results} estimation accuracy results are averaged over time and the leading $32$ eigenvectors, i.e., $\frac{1}{32T} \sum_{i=1}^{32} \sum_{t=2}^{T+1} \psi_{i}^{(t)}$. Fig. \ref{fig:hyperparam-sensitivity} shows approximation performance and run time of $\proposed_{\rm RSVD}$ with respect to the performance of $\proposed_3$. As $L$ and $P$ increase, so does the runtime of $\proposed_{\rm RSVD}$. At the same time, eigenvector approximation quality approaches that of $\proposed_3$.  Overall, the results highlight a clear complexity–accuracy trade-off: larger values of $L$ and $P$ improve approximation quality but increase computational cost, while moderate values already achieve near-identical accuracy with substantially lower runtime.

\subsection{Central Node Identification}

The next set of numerical tests focuses on the performance of downstream tasks using the eigenembeddings returned by the proposed algorithm. A fundamental graph mining task is identifying which nodes of a graph are important, or ``central''. The importance of nodes is commonly quantified via so-called centrality measures, which are used to rank nodes based on their importance within the graph topology. Although a plethora of measures are developed in literature, this section focuses on subgraph centrality since it utilizes both the leading eigenvalues and corresponding eigenvectors of the graph adjacency matrix~\cite{estrada2005subgraph,kalantzis2025single}. As before, let $\widetilde{\mLambda}_K^{(t)}$ be the diagonal matrix of $K$ leading eigenvalues, estimated by any of the considered methods for time point $t$, and let $\widetilde{\mX}_K^{(t)}$ be the matrix containing the corresponding eigenvectors. Subgraph centrality is approximated via the vector $\exp(\mA^{(t)})\mathbf{1}\approx\exp( \widetilde{\mX}_K^{(t)} \widetilde{\mLambda}_K^{(t)} \widetilde{\mX}_K^{(t)}{}^\top)  \mathbf{1} = \widetilde{\mX}_K^{(t)}\exp(\widetilde{\mLambda}_K^{(t)})\widetilde{\mX}_K^{(t)}{}^\top\mathbf{1} \in \setR^{N^{(t)}}$. Entries of this vector correspond to nodes of the graph, and entries with larger values indicate more ``central'' nodes. 

The algorithms under consideration are applied to dynamic graphs constructed in Scenario 1 of Sec. \ref{ssec:eigs-est-results}. For each time point $t$ a set $\widehat{\calI}^{(t)}$ of the $J$ most central nodes is identified using the estimated leading $32$ eigenpairs. As a baseline, a set $\calI^{(t)}$ is found using the reference leading eigenpairs returned by \verb|eigs| function. The performance of all algorithms in identifying central nodes is then measured as the average of $|\widetilde{\calI}^{(t)} \cap \calI^{(t)}|/J$ over $t$. Results are reported in Table \ref{tab:centrality} for all datasets and and two values of $J$, $100$ and $1000$. The results closely follow observations from previous numerical tests. TIMERS provides the best performance, while $\proposed_3$ and $\proposed_{\rm RSVD}$ exhibit the second best performance, with reduced computational cost. Among remaining algorithms, TRIP and RM have the worst performance, while IASC and $\proposed_2$ have the same results, inline with Scenario 1 of Sec. \ref{ssec:eigs-est-results}. 

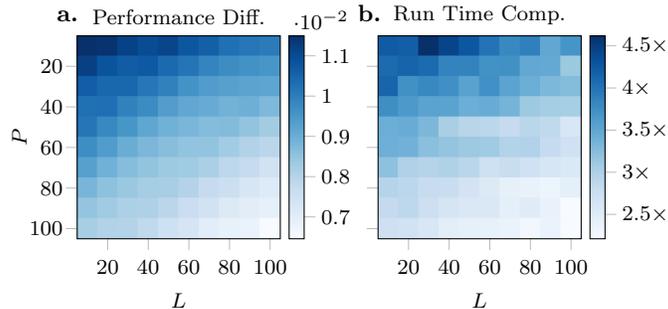
\begin{figure}[t]
    \centering
    \begin{tikzpicture}
    \begin{groupplot}[
        group style={
            group size=2 by 1,
            horizontal sep=1.3cm,
            ylabels at=edge left,
            yticklabels at=edge left,
        },
        width=2.7cm,
        height=2.7cm,
        colormap/Blues,
        tick align=outside,
        tick label style={font=\scriptsize},
        tick pos=left,
        yticklabel style={xshift=0.25em},
        xlabel={$L$},
        ylabel={$P$},
        label style={font=\scriptsize},
        title style={yshift=-0.5em, font=\scriptsize},
    ]

    \nextgroupplot[
        title={Performance Diff.},
        enlargelimits={abs=5},
        axis on top,
        scale only axis,
        grid=none,
        ylabel style={yshift=-0.6em},
        colorbar right,
        colorbar/width=0.2cm,
        colorbar style={
            tick pos=right,
            xshift=-0.4cm,
            try min ticks=4,
            max space between ticks={15},
            y tick scale label style={xshift=2em},
            yticklabel style={
                xshift=-0.4em,
                /pgf/number format/fixed,
                /pgf/number format/precision=1,
            },
        },
        axis equal,
        xmin=10, xmax=100,
        ymin=10, ymax=100,
        after end axis/.code={
            \node[anchor=south west, font=\bfseries\footnotesize] at (rel axis cs:-0.15,1.02) {a.};
        },
    ]
    \addplot[
        matrix plot,
        point meta=explicit,
        mesh/cols=10,
    ] table [meta=Value, col sep=comma, x=L, y=P] {fig/data-final/param-sens-cm-collab-0.50-20-32-perf.csv};

    \nextgroupplot[
        title={Run Time Comp.},
        enlargelimits={abs=5},
        axis on top,
        scale only axis,
        grid=none,
        colorbar right,
        colorbar/width=0.2cm,
        colorbar style={
            tick pos=right,
            xshift=-0.4cm,
            ytick distance=0.5,
            yticklabel={\pgfmathparse{\tick}\pgfmathprintnumber{\pgfmathresult}$\times$},
            yticklabel style={/pgf/number format/1000 sep={}, xshift=-0.4em},
        },
        axis equal,
        xmin=10, xmax=100,
        ymin=10, ymax=100,
        after end axis/.code={
            \node[anchor=south west, font=\bfseries\footnotesize] at (rel axis cs:-0.15,1.02) {b.};
        },
    ]
    \addplot [
        matrix plot,
        point meta=explicit,
        mesh/cols=10,
    ] table [meta=Value, col sep=comma, x=L, y=P] {fig/data-final/param-sens-cm-collab-0.50-20-32-rt.csv};
    \end{groupplot}
\end{tikzpicture}
    \vspace{-1em}
    \caption{Effect of $L$ and $P$ parameters on eigenvector approximation and run time of $\proposed_{\rm RSVD}$, compared to $\proposed_3$. (a). plots the performance difference between $\proposed_{\rm RSVD}$ and $\proposed_3$, whose performance is $0.7 \cdot 10^{-2}$. (b). plots the ratio of $\proposed_3$'s run time to that of $\proposed_{\rm RSVD}$, i.e., how many times $\proposed_{\rm RSVD}$ is faster than $\proposed_3$.}
    \label{fig:hyperparam-sensitivity}
\end{figure}

\begin{table}[b]
\caption{Accuracy of detecting central nodes via subgraph centrality}
\centering
\begin{tabular}{llcccc}
\toprule
& & \multicolumn{4}{c}{Datasets} \\ \cmidrule{3-6}
Method & $J$ & Crocodile & CM-Collab & Epinions & Twitch \\ \midrule
TRIP & 100 & $91.0\%$ & $96.9\%$ & $98.0\%$ & $97.2\%$ \\ 
RM & 100 & $91.8\%$ & $96.8\%$ & $98.0\%$ & $96.6\%$ \\ 
IASC & 100 & $95.1\%$ & $98.0\%$ & $99.1\%$ & $97.3\%$ \\ 
TIMERS & 100 & $97.6\%$ & $99.3\%$ & $99.9\%$ & $-$ \\ 
$\proposed_2$ & 100 & $95.1\%$ & $98.0\%$ & $99.1\%$ & $97.3\%$ \\ 
$\proposed_3$ & 100 & $99.3\%$ & $99.0\%$ & $99.3\%$ & $98.9\%$ \\ 
$\proposed_{\rm RSVD}$ & 100 & $99.2\%$ & $98.6\%$ & $99.2\%$ & $98.0\%$ \\ 
\midrule
TRIP & 1000 & $97.8\%$ & $98.5\%$ & $99.4\%$ & $97.6\%$ \\ 
RM & 1000 & $98.1\%$ & $98.6\%$ & $99.4\%$ & $97.8\%$ \\ 
IASC & 1000 & $98.5\%$ & $99.6\%$ & $99.8\%$ & $99.1\%$ \\ 
TIMERS & 1000 & $99.1\%$ & $99.8\%$ & $99.9\%$ & $-$ \\ 
$\proposed_2$ & 1000 & $98.5\%$ & $99.6\%$ & $99.8\%$ & $99.1\%$ \\ 
$\proposed_3$ & 1000 & $98.8\%$ & $99.8\%$ & $99.9\%$ & $99.6\%$ \\ 
$\proposed_{\rm RSVD}$ & 1000 & $98.8\%$ & $99.7\%$ & $99.9\%$ & $99.4\%$ \\ 
\bottomrule
\end{tabular}
\label{tab:centrality}
\end{table}

\subsection{Graph Clustering}

Another key task in machine learning is to partition nodes of a graph $\calG$ into $K$ groups or clusters, such that nodes within a cluster are densely connected to each other while having few connections with the nodes in different clusters. Let $\calP_k$ denote the set of nodes that belong to the $k$th cluster. Then a partition of the nodes is denoted as $\frakP = \{\calP_k\}_{k=1}^K$.  A widely used approach for grouping nodes of a graph is spectral clustering~\cite{von2007tutorial}, which exploits the spectral properties of the Laplacian $\mathbf{L}$ or  normalized Laplacian matrix $\mathbf{L}_{\rm n}$. In the ideal case where there are no edges between clusters, the eigenvectors corresponding to the $K$ smallest eigenvalues of $\mathbf{L}_{\rm n}$ are indicator vectors of $\calP_k$'s. In practice when there are relatively few inter-cluster edges, these eigenvectors still provide an embedding that is informative to infer $\frakP$. Upon computing the eigenvectors associated with the $K$ smallest eigenvalues and collecting them as columns of $\mathbf{X}_{K}$, the rows of $\mathbf{X}_K$ are clustered using $K$-means~\cite{lloyd1982least}. As each row of $\mathbf{X}_K$ corresponds to a node in $\cal G$, this procedure yields an estimated partition $\widehat{\frakP}$. 

For a dynamic graph $\{\calG^{(t)}\}_{t=0}^T$, clustering consists of finding $\frakP^{(t)}$ for each $t$ where $\frakP^{(t)} = \{P_k^{(t)}\}_{r=1}^{K}$ is the partitioning of $\calG^{(t)}$ into $K$ clusters. This section benchmarks different tracking algorithms in terms of how well their eigenvector estimates can infer $\frakP^{(t)}$'s. Let $\widetilde{\mX}_{K}^{(t)}$ be an estimate of eigenvectors corresponding to smallest $K$ eigenvalues of $\mL_{\rm n}^{(t)}$ and let $\widehat{\frakP}^{(t)}$ be the clusters obtained by applying $K$-means to $\widetilde{\mX}_{K}^{(t)}$. The performance of different algorithms is measured via $ARI(\widehat{\frakP}^{(t)}, \frakP^{(t)})$ where $ARI$ denotes the adjusted rand index that measures the similarity between $\widehat{\frakP}^{(t)}$ and ground truth $\frakP^{(t)}$ \cite{hubert1985comparing}. 

\begin{figure}
    \centering
    \hspace{2em} \ref{leg:exp6-aris} \\[0.2em]
%
%
%
\begin{tikzpicture}
\begin{groupplot}[%
    group style={
        group size=2 by 1,
        horizontal sep=0.5cm,
        ylabels at=edge left,
        yticklabels at=edge left,
    },
    width=5cm,
    xmajorgrids,
    ymajorgrids,
    tick align=outside,
    tick label style={font=\scriptsize},
    tick pos=left,
    yticklabel style={xshift=0.25em},
    ylabel={ARI Ratio},
    label style={font=\scriptsize},
    ylabel style={yshift=-0.25em},
    xlabel style={yshift=0.25em},
    legend style={
        anchor=center,
        at={(0.5, 0.5)},
        font=\scriptsize,
        legend columns=4,
        /tikz/every even column/.append style={
            column sep=0.2em,
        }
    },
    legend cell align={left},
    enlarge x limits=0.05,
    enlarge y limits=0.05,
    ytick distance=0.1,
    ymin=0.7,
    ymax=1.0,
]

\nextgroupplot[
    title={ARI vs Inter-cluster Probability},
    scaled x ticks=false,
    xtick distance=0.002,
    xticklabel style={
        /pgf/number format/.cd,
        /pgf/number format/fixed, 
        /pgf/number format/precision=4,
        skip 0.,
    },
    xlabel={$p_{\rm out}$},
    title style={yshift=-0.5em, font=\scriptsize},
    after end axis/.code={
        \node[anchor=south west, font=\bfseries\footnotesize] at (rel axis cs:-0.15,1.02) {a.};
    },
]

\addplot[color=T-Q-V1, mark=*, MyLinePlot, densely dashed] table[
    col sep=comma, x=p_out, y expr={\thisrow{trip} / \thisrow{eigs}}
] {fig/data-final/exp6-aris-vs-pout.csv};

\addplot[color=T-Q-V2, mark=square*, MyLinePlot, densely dashed] table[
    col sep=comma, x=p_out, y expr={\thisrow{residual-mode} / \thisrow{eigs}}
] {fig/data-final/exp6-aris-vs-pout.csv};

\addplot[color=T-Q-MC7, mark=triangle*, MyLinePlot, densely dashed] table[
    col sep=comma, x=p_out, y expr={\thisrow{iasc} / \thisrow{eigs}}
] {fig/data-final/exp6-aris-vs-pout.csv};

\addplot[color=T-Q-V6, mark=diamond*, MyLinePlot, densely dashed] table[
    col sep=comma, x=p_out, y expr={\thisrow{timers} / \thisrow{eigs}}
] {fig/data-final/exp6-aris-vs-pout.csv};

\addplot[color=T-Q-V3, mark=*, MyLinePlot] table[
    col sep=comma, x=p_out, y expr={\thisrow{grest-2} / \thisrow{eigs}}
] {fig/data-final/exp6-aris-vs-pout.csv};

\addplot[color=T-Q-V4, mark=square*, MyLinePlot] table[
    col sep=comma, x=p_out, y expr={\thisrow{grest-3} / \thisrow{eigs}}
] {fig/data-final/exp6-aris-vs-pout.csv};

\addplot[color=T-Q-V5, mark=triangle*, MyLinePlot] table[
    col sep=comma, x=p_out, y expr={\thisrow{grest-rsvd-20-20} / \thisrow{eigs}}
] {fig/data-final/exp6-aris-vs-pout.csv};

\nextgroupplot[
    title={ARI vs Number of Clusters},
    legend to name=leg:exp6-aris,
    xtick distance=2,
    xlabel={$K$},
    title style={yshift=-0.3em, font=\scriptsize},
    after end axis/.code={
        \node[anchor=south west, font=\bfseries\footnotesize] at (rel axis cs:-0.15,1.02) {b.};
    },
]

\addplot[color=T-Q-V1, mark=*, MyLinePlot, densely dashed] table[
    col sep=comma, x=n_cl, y expr={\thisrow{trip} / \thisrow{eigs}}
] {fig/data-final/exp6-aris-vs-ncl.csv};
\addlegendentry{TRIP}

\addplot[color=T-Q-V2, mark=square*, MyLinePlot, densely dashed] table[
    col sep=comma, x=n_cl, y expr={\thisrow{residual-mode} / \thisrow{eigs}}
] {fig/data-final/exp6-aris-vs-ncl.csv};
\addlegendentry{RM}

\addplot[color=T-Q-MC7, mark=triangle*, MyLinePlot, densely dashed] table[
    col sep=comma, x=n_cl, y expr={\thisrow{iasc} / \thisrow{eigs}}
] {fig/data-final/exp6-aris-vs-ncl.csv};
\addlegendentry{IASC}

\addplot[color=T-Q-V6, mark=diamond*, MyLinePlot, densely dashed] table[
    col sep=comma, x=n_cl, y expr={\thisrow{timers} / \thisrow{eigs}}
] {fig/data-final/exp6-aris-vs-ncl.csv};
\addlegendentry{TIMERS}

\addplot[color=T-Q-V3, mark=*, MyLinePlot] table[
    col sep=comma, x=n_cl, y expr={\thisrow{grest-2} / \thisrow{eigs}}
] {fig/data-final/exp6-aris-vs-ncl.csv};
\addlegendentry{$\proposed_{2}$}

\addplot[color=T-Q-V4, mark=square*, MyLinePlot] table[
    col sep=comma, x=n_cl, y expr={\thisrow{grest-3} / \thisrow{eigs}}
] {fig/data-final/exp6-aris-vs-ncl.csv};
\addlegendentry{$\proposed_{3}$}

\addplot[color=T-Q-V5, mark=triangle*, MyLinePlot] table[
    col sep=comma, x=n_cl, y expr={\thisrow{grest-rsvd-20-20} / \thisrow{eigs}}
] {fig/data-final/exp6-aris-vs-ncl.csv};
\addlegendentry{$\proposed_{\rm RSVD}$}

\end{groupplot}
\end{tikzpicture}%
    \caption{Clustering performance on synthetic dynamic graphs with clustering structure. a. and b. show the performance with respect to inter-cluster edge probability $p_{\rm out}$ and number of clusters $R$, respectively.}
    \label{fig:sbm-clustering}
\end{figure}
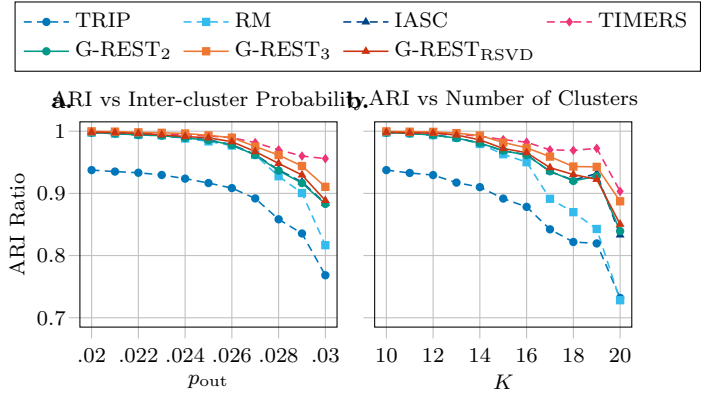

As the real datasets considered in previous experiments do not include clustering information of the nodes, the proposed method is tested on synthetic dynamic graphs generated using stochastic block models (SBM) where the ground truth $\frakP^{(t)}$'s are known. An SBM graph $\calG=(\calV, \calE)$ with $N$ nodes and $K$ clusters is generated by randomly assigning each node to a cluster $\calP_k$. Nodes $i$ and $j$ are connected with probability $p_{\rm in}$ if $i$ and $j$ are in the same cluster and they are connected with probability $p_{\rm out}$, otherwise. For a graph that exhibits clustering structure, it holds that $p_{\rm out} < p_{\rm in}$. Once $\calG$ is generated, a dynamic graph is created using the procedure outlined in previous experiments. Namely, a node set $\calV^{(0)}$ with cardinality $N^{(0)}$ is randomly selected from $\calV$ and $\calG^{(0)}$ is set as the subgraph induced by $\calV^{(0)}$. $\calG^{(t)}$ for $T \geq t > 0$ is then determined via graph expansion, i.e., it is the subgraph induced by $\calV^{(t)} = \calV^{(t-1)} \cup \calC^{(t)}$ where $\calC^{(t)}$ is the set of $S^{(t)}$ nodes randomly drawn from $\calV \setminus \calV^{(t-1)}$. In the following, $N=10,000$, $p_{\rm in}=0.05$, $N^{(0)}=9500$, $T=10$ and $C^{(t)} = 50$; while different values of $K$ and $p_{\rm out}$ are tested to observe their effects on clustering performance. Note that, since the size of the graph expansion is $C^{(t)}=50$, the parameters of $\proposed_{\rm RSVD}$ are reduced to $L=P=20$. 

Figs. \ref{fig:sbm-clustering}(a) and \ref{fig:sbm-clustering}(b) show the clustering performance of considered algorithms for varying values of $p_{\rm out}$ and $K$, respectively. The reported performance is with respect to the performance of reference eigenvectors obtained by \verb|eigs|, that is $\frac{1}{T} \sum_{t=1}^T \frac{ARI(\widehat{\frakP}_{\rm m}^{(t)}, \frakP^{(t)})}{ARI(\widehat{\frakP}_{\rm e}^{(t)}, \frakP^{(t)})}$ where $\widehat{\frakP}_{\rm e}^{(t)}$ and $\widehat{\frakP}_{\rm m}^{(t)}$ are the clusters found for time $t$ via the eigenvectors estimated via \verb|eigs| and those estimated by a tracking method, respectively. The results are inline with previous experiments; TIMERS exhibits the best performance, closely followed by $\proposed_{3}$. $\proposed_{\rm RSVD}$ is slightly better than $\proposed_2$ and IASC, which have the same performance. This is in line with observations from previous experiments where the updates only include graph expansion. RM and TRIP exhibit the worst ARI ratio with the former being superior to TRIP. Finally, increasing $K$ and $p_{\rm out}$ results in graphs with harder to detect clusters, yielding worse performance across all algorithms.

\section{Conclusions and future work}
\label{sec:conclusion}
The present paper introduced a novel algorithm for tracking the dominant eigenspace of an evolving graph. The proposed algorithm can handle both addition and removal of edges and addition of new nodes, while maintaining competitive computational and memory complexity. Performance of the novel algorithm was compared against state-of-the-art benchmarks on several graph datasets. Future work will focus on robust variants of the proposed methods, as well as methods that can handle removal of nodes, additional graph applications such as graph filtering for dynamic graphs, and distributed and federated variants of the proposed algorithm.
Another interesting direction is to consider the proposed algorithm to update the eigenembeddings of graph neural networks  when the underlying graph changes dynamically.

\bibliographystyle{ieeetr}
\bibliography{refs}

@book{byron1992mathematics,
  title = {Mathematics of Classical and Quantum Physics},
  author = {Byron, Frederick W. and Fuller, Robert W.},
  year = {1992},
  publisher = {Dover Publications},
  address = {New York},
  isbn = {978-0-486-67164-2}
}

@book{kato2013perturbation,
  title={Perturbation theory for linear operators},
  author={Kato, Tosio},
  volume={132},
  year={2013},
  publisher={Springer Science \& Business Media}
}

@book{wilkinson1965algebraic,
  title={The algebraic eigenvalue problem},
  author={Wilkinson, James Hardy},
  volume={662},
  year={1965},
  publisher={Oxford Clarendon}
}

@book{rellich1969perturbation,
  title={Perturbation theory of eigenvalue problems},
  author={Rellich, Franz and Berkowitz, Joseph},
  year={1969},
  publisher={CRC Press}
}

@inproceedings{chen2015fast,
  title = {Fast {{Eigen-Functions Tracking}} on {{Dynamic Graphs}}},
  booktitle = {Proceedings of the 2015 {{SIAM International Conference}} on {{Data Mining}}},
  author = {Chen, Chen and Tong, Hanghang},
  year = 2015,
  month = jun,
  pages = {559--567},
  publisher = {{Society for Industrial and Applied Mathematics}},
  doi = {10.1137/1.9781611974010.63},
  isbn = {978-1-61197-401-0}
}

@article{chen2017eigenfunctions,
  title = {On the Eigen-functions of Dynamic Graphs: {{Fast}} Tracking and Attribution Algorithms},
  shorttitle = {On the Eigen-functions of Dynamic Graphs},
  author = {Chen, Chen and Tong, Hanghang},
  year = 2017,
  month = apr,
  journal = {Statistical Analysis and Data Mining: The ASA Data Science Journal},
  volume = {10},
  number = {2},
  pages = {121--135},
  issn = {1932-1864, 1932-1872},
  doi = {10.1002/sam.11310},
  copyright = {http://onlinelibrary.wiley.com/termsAndConditions\#am}
}

@article{zha1999updating,
  title = {On {{Updating Problems}} in {{Latent Semantic Indexing}}},
  author = {Zha, Hongyuan and Simon, Horst D.},
  year = 1999,
  month = jan,
  journal = {SIAM Journal on Scientific Computing},
  volume = {21},
  number = {2},
  pages = {782--791},
  publisher = {{Society for Industrial and Applied Mathematics}},
  issn = {1064-8275},
  doi = {10.1137/S1064827597329266}
}

@article{dhanjal2014efficient,
  title = {Efficient Eigen-Updating for Spectral Graph Clustering},
  author = {Dhanjal, Charanpal and Gaudel, Romaric and Cl{\'e}men{\c c}on, St{\'e}phan},
  year = 2014,
  month = may,
  journal = {Neurocomputing},
  volume = {131},
  pages = {440--452},
  issn = {09252312},
  doi = {10.1016/j.neucom.2013.11.015}
}

@article{deng2024efficient,
  title = {Efficient {{Eigen-Decomposition}} for {{Low-Rank Symmetric Matrices}} in {{Graph Signal Processing}}: {{An Incremental Approach}}},
  shorttitle = {Efficient {{Eigen-Decomposition}} for {{Low-Rank Symmetric Matrices}} in {{Graph Signal Processing}}},
  author = {Deng, Qinwen and Zhang, Yangwen and Li, Mo and Zhang, Songyang and Ding, Zhi},
  year = 2024,
  month = aug,
  journal = {IEEE Transactions on Signal Processing},
  volume = {72},
  pages = {4918--4934},
  issn = {1053-587X, 1941-0476},
  doi = {10.1109/TSP.2024.3438676},
  copyright = {https://ieeexplore.ieee.org/Xplorehelp/downloads/license-information/IEEE.html}
}

@article{mitz2022perturbation,
  title={A Perturbation-Based Kernel Approximation Framework},
  author={Mitz, Roy and Shkolnisky, Yoel},
  journal={Journal of Machine Learning Research},
  volume={23},
  number={142},
  pages={1--26},
  year={2022}
}

@article{mitz2019symmetric,
  title={Symmetric rank-one updates from partial spectrum with an application to out-of-sample extension},
  author={Mitz, Roy and Sharon, Nir and Shkolnisky, Yoel},
  journal={SIAM Journal on Matrix Analysis and Applications},
  volume={40},
  number={3},
  pages={973--997},
  year={2019},
  publisher={SIAM}
}

@inproceedings{zhang2017timers,
  title={Timers: Error-bounded svd restart on dynamic networks},
  author={Zhang, Ziwei and Cui, Peng and Pei, Jian and Wang, Xiao and Zhu, Wenwu},
  booktitle={Proceedings of the 2017 international conference on management of data},
  pages={221--236},
  year={2017}
}

@inproceedings{kang2011centralities,
  title={Centralities in large graphs: Algorithms and next steps},
  author={Kang, U and Meeder, Brendan and Papalexakis, Evangelos E and Faloutsos, Christos},
  booktitle={Proceedings of the 17th ACM SIGKDD International Conference on Knowledge Discovery and Data Mining},
  pages={2--2},
  year={2011},
  organization={ACM}
}

@misc{snapnets,
  author       = {Jure Leskovec and Andrej Krevl},
  title        = {{SNAP Datasets}: {Stanford} Large Network Dataset Collection},
  howpublished = {\url{http://snap.stanford.edu/data}},
  month        = jun,
  year         = 2014
}

@article{farkas_spectra_2001,
  title = {Spectra of ``real-world'' graphs: Beyond the semicircle law},
  author = {Farkas, Ill\'es J. and Der\'enyi, Imre and Barab\'asi, Albert-L\'aszl\'o and Vicsek, Tam\'as},
  journal = {Phys. Rev. E},
  volume = {64},
  issue = {2},
  pages = {026704},
  numpages = {12},
  year = {2001},
  month = {Jul},
  publisher = {American Physical Society},
  doi = {10.1103/PhysRevE.64.026704},
  url = {https://link.aps.org/doi/10.1103/PhysRevE.64.026704}
}

@article{vecharynski2014fast,
  title={Fast updating algorithms for latent semantic indexing},
  author={Vecharynski, Eugene and Saad, Yousef},
  journal={SIAM Journal on Matrix Analysis and Applications},
  volume={35},
  number={3},
  pages={1105--1131},
  year={2014},
  publisher={SIAM}
}

@article{spectralclustering,
  title={A tutorial on spectral clustering},
  author={Von Luxburg, Ulrike},
  journal={Statistics and computing},
  volume={17},
  number={4},
  pages={395--416},
  year={2007},
  publisher={Springer}
}

@article{vatrapu2016social,
  title={Social set analysis: A set theoretical approach to big data analytics},
  author={Vatrapu, Ravi and Mukkamala, Raghava Rao and Hussain, Abid and Flesch, Benjamin},
  journal={Ieee Access},
  volume={4},
  pages={2542--2571},
  year={2016},
  publisher={IEEE}
}

@article{estrada2005subgraph,
  title={Subgraph centrality in complex networks},
  author={Estrada, Ernesto and Rodriguez-Velazquez, Juan A},
  journal={Physical Review E},
  volume={71},
  number={5},
  pages={056103},
  year={2005},
  publisher={APS}
}

@article{simon1984lanczos,
  title={The Lanczos algorithm with partial reorthogonalization},
  author={Simon, Horst D},
  journal={Mathematics of computation},
  volume={42},
  number={165},
  pages={115--142},
  year={1984}
}

@article{moonen1992singular,
  title={A singular value decomposition updating algorithm for subspace tracking},
  author={Moonen, Marc and Van Dooren, Paul and Vandewalle, Joos},
  journal={SIAM Journal on Matrix Analysis and Applications},
  volume={13},
  number={4},
  pages={1015--1038},
  year={1992},
  publisher={SIAM}
}

@article{hall2002adding,
  title={Adding and subtracting eigenspaces with eigenvalue decomposition and singular value decomposition},
  author={Hall, Peter and Marshall, David and Martin, Ralph},
  journal={Image and Vision Computing},
  volume={20},
  number={13-14},
  pages={1009--1016},
  year={2002},
  publisher={Elsevier}
}

@inproceedings{brand2003fast,
  title={Fast online svd revisions for lightweight recommender systems},
  author={Brand, Matthew},
  booktitle={Proceedings of the 2003 SIAM international conference on data mining},
  pages={37--46},
  year={2003},
  organization={SIAM}
}

@article{degroat2002efficient,
  title={Efficient, numerically stabilized rank-one eigenstructure updating (signal processing)},
  author={DeGroat, Ronald D and Roberts, Richard A},
  journal={IEEE Transactions on acoustics, speech, and signal processing},
  volume={38},
  number={2},
  pages={301--316},
  year={2002},
  publisher={IEEE}
}

@inproceedings{rossi2015network,
  title={The network data repository with interactive graph analytics and visualization},
  author={Rossi, Ryan and Ahmed, Nesreen},
  booktitle={Proceedings of the AAAI conference on artificial intelligence},
  volume={29},
  number={1},
  year={2015}
}

@inproceedings{brand2002incremental,
  title={Incremental singular value decomposition of uncertain data with missing values},
  author={Brand, Matthew},
  booktitle={European Conference on Computer Vision},
  pages={707--720},
  year={2002},
  organization={Springer}
}

@article{hubert1985comparing,
  title={Comparing partitions},
  author={Hubert, Lawrence and Arabie, Phipps},
  journal={Journal of classification},
  volume={2},
  number={1},
  pages={193--218},
  year={1985},
  publisher={Springer}
}

@article{cai2018comprehensive,
  title={A comprehensive survey of graph embedding: Problems, techniques, and applications},
  author={Cai, Hongyun and Zheng, Vincent W and Chang, Kevin Chen-Chuan},
  journal={IEEE transactions on knowledge and data engineering},
  volume={30},
  number={9},
  pages={1616--1637},
  year={2018},
  publisher={IEEE}
}

@article{xu2021understanding,
  title={Understanding graph embedding methods and their applications},
  author={Xu, Mengjia},
  journal={SIAM Review},
  volume={63},
  number={4},
  pages={825--853},
  year={2021},
  publisher={SIAM}
}

@ARTICLE{gsp_spm_2023,
  author={Leus, Geert and Marques, Antonio G. and Moura, José M.F. and Ortega, Antonio and Shuman, David I},
  journal={IEEE Signal Processing Magazine}, 
  title={Graph Signal Processing: History, development, impact, and outlook}, 
  year={2023},
  volume={40},
  number={4},
  pages={49-60},
  doi={10.1109/MSP.2023.3262906}}

@article{lovasz1993random,
  title={Random walks on graphs},
  author={Lov{\'a}sz, L{\'a}szl{\'o}},
  journal={Combinatorics, Paul erdos is eighty},
  volume={2},
  number={1-46},
  pages={4},
  year={1993}
}

@inproceedings{kollias2022directed,
  title={Directed graph auto-encoders},
  author={Kollias, Georgios and Kalantzis, Vasileios and Id{\'e}, Tsuyoshi and Lozano, Aur{\'e}lie and Abe, Naoki},
  booktitle={Proceedings of the AAAI conference on artificial intelligence},
  volume={36},
  number={7},
  pages={7211--7219},
  year={2022}
}

@article{chi2013petrels,
  title={Petrels: Parallel subspace estimation and tracking by recursive least squares from partial observations},
  author={Chi, Yuejie and Eldar, Yonina C and Calderbank, Robert},
  journal={IEEE Transactions on Signal Processing},
  volume={61},
  number={23},
  pages={5947--5959},
  year={2013},
  publisher={IEEE}
}

@book{higham2016matlab,
  title={MATLAB guide},
  author={Higham, Desmond J and Higham, Nicholas J},
  year={2016},
  publisher={SIAM}
}

@article{kalantzis2025single,
  title={Single-pass top-N subgraph centrality of graphs via subspace projections},
  author={Kalantzis, Vassilis and Kollias, Georgios and Ubaru, Shashanka and Abe, Naoki and Horesh, Lior},
  journal={Journal of Complex Networks},
  volume={13},
  number={1},
  pages={cnae049},
  year={2025},
  publisher={Oxford University Press}
}

@article{marenco2022online,
  title={Online change point detection for weighted and directed random dot product graphs},
  author={Marenco, Bernardo and Bermolen, Paola and Fiori, Marcelo and Larroca, Federico and Mateos, Gonzalo},
  journal={IEEE Transactions on Signal and Information Processing over Networks},
  volume={8},
  pages={144--159},
  year={2022},
  publisher={IEEE}
}

@article{fiori2023gradient,
  title={Gradient-based spectral embeddings of random dot product graphs},
  author={Fiori, Marcelo and Marenco, Bernardo and Larroca, Federico and Bermolen, Paola and Mateos, Gonzalo},
  journal={IEEE Transactions on Signal and Information Processing over Networks},
  volume={10},
  pages={1--16},
  year={2023},
  publisher={IEEE}
}

@article{abed2002fast,
  title={Fast orthonormal PAST algorithm},
  author={Abed-Meraim, Karim and Chkeif, Ammar and Hua, Yingbo},
  journal={IEEE Signal processing letters},
  volume={7},
  number={3},
  pages={60--62},
  year={2002},
  publisher={IEEE}
}

@article{badeau2005fast,
  title={Fast approximated power iteration subspace tracking},
  author={Badeau, Roland and David, Bertrand and Richard, Ga{\"e}l},
  journal={IEEE Transactions on Signal Processing},
  volume={53},
  number={8},
  pages={2931--2941},
  year={2005},
  publisher={IEEE}
}

@ARTICLE{isufi_tsp_graphfilters_2024,
  author={Isufi, Elvin and Gama, Fernando and Shuman, David I and Segarra, Santiago},
  journal={IEEE Transactions on Signal Processing}, 
  title={Graph Filters for Signal Processing and Machine Learning on Graphs}, 
  year={2024},
  volume={72},
  number={},
  pages={4745-4781},
  doi={10.1109/TSP.2024.3349788}}

@article{jmlr_randomdotproduct_2018,
  author  = {Avanti Athreya and Donniell E. Fishkind and Minh Tang and Carey E. Priebe and Youngser Park and Joshua T. Vogelstein and Keith Levin and Vince Lyzinski and Yichen Qin and Daniel L Sussman},
  title   = {Statistical Inference on Random Dot Product Graphs: a Survey},
  journal = {Journal of Machine Learning Research},
  year    = {2018},
  volume  = {18},
  number  = {226},
  pages   = {1--92},
  url     = {http://jmlr.org/papers/v18/17-448.html}
}

@ARTICLE{gd_based_spectral_embeddings_TSP24,
  author={Fiori, Marcelo and Marenco, Bernardo and Larroca, Federico and Bermolen, Paola and Mateos, Gonzalo},
  journal={IEEE Transactions on Signal and Information Processing over Networks}, 
  title={Gradient-Based Spectral Embeddings of Random Dot Product Graphs}, 
  year={2024},
  volume={10},
  number={},
  pages={1-16},
  doi={10.1109/TSIPN.2023.3343607}}

@misc{marenco2025weightedrandomdotproduct,
      title={Weighted Random Dot Product Graphs}, 
      author={Bernardo Marenco and Paola Bermolen and Marcelo Fiori and Federico Larroca and Gonzalo Mateos},
      year={2025},
      eprint={2505.03649},
      archivePrefix={arXiv},
      primaryClass={stat.ML},
      url={https://arxiv.org/abs/2505.03649}, 
}

@InProceedings{smola_kondor_graphkernels_2003,
author="Smola, Alexander J.
and Kondor, Risi",
editor="Sch{\"o}lkopf, Bernhard
and Warmuth, Manfred K.",
title="Kernels and Regularization on Graphs",
booktitle="Learning Theory and Kernel Machines",
year="2003",
publisher="Springer Berlin Heidelberg",
address="Berlin, Heidelberg",
pages="144--158",
isbn="978-3-540-45167-9"
}

@article{dynamic_network_embedding_survey_2022,
author = {Xue, Guotong and Zhong, Ming and Li, Jianxin and Chen, Jia and Zhai, Chengshuai and Kong, Ruochen},
title = {Dynamic network embedding survey},
year = {2022},
issue_date = {Feb 2022},
publisher = {Elsevier Science Publishers B. V.},
address = {NLD},
volume = {472},
number = {C},
issn = {0925-2312},
url = {https://doi.org/10.1016/j.neucom.2021.03.138},
doi = {10.1016/j.neucom.2021.03.138},
journal = {Neurocomput.},
month = feb,
pages = {212–223},
numpages = {12}
}

@ARTICLE{network_representation_survey_2020,
  author={Zhang, Daokun and Yin, Jie and Zhu, Xingquan and Zhang, Chengqi},
  journal={IEEE Transactions on Big Data}, 
  title={Network Representation Learning: A Survey}, 
  year={2020},
  volume={6},
  number={1},
  pages={3-28},
  doi={10.1109/TBDATA.2018.2850013}}

@article{Chen_Wang_Wang_Kuo_2020, title={Graph representation learning: a survey}, volume={9}, DOI={10.1017/ATSIP.2020.13}, journal={APSIPA Transactions on Signal and Information Processing}, author={Chen, Fenxiao and Wang, Yun-Cheng and Wang, Bin and Kuo, C.-C. Jay}, year={2020}, pages={e15}}

@article{halko2011finding,
  title={Finding structure with randomness: Probabilistic algorithms for constructing approximate matrix decompositions},
  author={Halko, Nathan and Martinsson, Per-Gunnar and Tropp, Joel A},
  journal={SIAM review},
  volume={53},
  number={2},
  pages={217--288},
  year={2011},
  publisher={SIAM}
}

@article{martinsson2020randomized,
  title={Randomized numerical linear algebra: Foundations and algorithms},
  author={Martinsson, Per-Gunnar and Tropp, Joel A},
  journal={Acta Numerica},
  volume={29},
  pages={403--572},
  year={2020},
  publisher={Cambridge University Press}
}

@inproceedings{kalantzis2023matrix,
  title = {Matrix {{Resolvent Eigenembeddings}} for {{Dynamic Graphs}}},
  booktitle = {{{ICASSP}} 2023 - 2023 {{IEEE International Conference}} on {{Acoustics}}, {{Speech}} and {{Signal Processing}} ({{ICASSP}})},
  author = {Kalantzis, Vasileios and Traganitis, Panagiotis A.},
  year = {2023},
  month = jun,
  pages = {1--5},
  publisher = {IEEE},
  address = {Rhodes Island, Greece},
  doi = {10.1109/ICASSP49357.2023.10096476},
  urldate = {2025-05-12},
  copyright = {https://doi.org/10.15223/policy-029},
  isbn = {978-1-7281-6327-7},
  langid = {english}
}

@inproceedings{kalantzis2021projection,
  title={Projection techniques to update the truncated SVD of evolving matrices with applications},
  author={Kalantzis, Vasileios and Kollias, Georgios and Ubaru, Shashanka and Nikolakopoulos, Athanasios N and Horesh, Lior and Clarkson, Kenneth},
  booktitle={International Conference on Machine Learning},
  pages={5236--5246},
  year={2021},
  organization={PMLR}
}

@book{demmel1997applied,
  title={Applied numerical linear algebra},
  author={Demmel, James W},
  year={1997},
  publisher={SIAM}
}

@book{golub2013matrix,
  title={Matrix computations},
  author={Golub, Gene H and Van Loan, Charles F},
  year={2013},
  publisher={JHU press}
}

@article{lloyd1982least,
  title={Least squares quantization in PCM},
  author={Lloyd, Stuart},
  journal={IEEE transactions on information theory},
  volume={28},
  number={2},
  pages={129--137},
  year={1982},
  publisher={IEEE}
}

@article{von2007tutorial,
  title={A tutorial on spectral clustering},
  author={Von Luxburg, Ulrike},
  journal={Statistics and computing},
  volume={17},
  pages={395--416},
  year={2007},
  publisher={Springer}
}

@article{stewart1990matrix,
  title={Matrix perturbation theory},
  author={Stewart, Gilbert W},
  year={1990},
  publisher={Citeseer}
}

@techreport{page1999pagerank,
  title={The {P}age{R}ank citation ranking: Bringing order to the web.},
  author={Page, Lawrence and Brin, Sergey and Motwani, Rajeev and Winograd, Terry},
  year={1999},
  institution={Stanford InfoLab}
}

@article{benzi2013ranking,
  title={Ranking hubs and authorities using matrix functions},
  author={Benzi, Michele and Estrada, Ernesto and Klymko, Christine},
  journal={Linear Algebra and its Applications},
  volume={438},
  number={5},
  pages={2447--2474},
  year={2013},
  publisher={Elsevier}
}

@article{evolutionary_network_analysis,
author = {Aggarwal, Charu and Subbian, Karthik},
title = {Evolutionary Network Analysis: A Survey},
year = {2014},
issue_date = {July 2014},
publisher = {Association for Computing Machinery},
address = {New York, NY, USA},
volume = {47},
number = {1},
issn = {0360-0300},
url = {https://doi.org/10.1145/2601412},
doi = {10.1145/2601412},
journal = {ACM Comput. Surv.},
month = may,
articleno = {10},
numpages = {36}
}

@ARTICLE{Bullmore2009-ag,
  title    = "Complex brain networks: graph theoretical analysis of structural
              and functional systems",
  author   = "Bullmore, Ed and Sporns, Olaf",
  journal  = "Nat Rev Neurosci",
  volume   =  10,
  number   =  3,
  pages    = "186--198",
  month    =  feb,
  year     =  2009,
  address  = "England",
  language = "en"
}

@ARTICLE{geometric_DL2017,
  author={Bronstein, Michael M. and Bruna, Joan and LeCun, Yann and Szlam, Arthur and Vandergheynst, Pierre},
  journal={IEEE Signal Processing Magazine}, 
  title={Geometric Deep Learning: Going beyond Euclidean data}, 
  year={2017},
  volume={34},
  number={4},
  pages={18-42},
  doi={10.1109/MSP.2017.2693418}}

@article{community_dynamic_networks2018,
author = {Rossetti, Giulio and Cazabet, R\'{e}my},
title = {Community Discovery in Dynamic Networks: A Survey},
year = {2018},
issue_date = {March 2019},
publisher = {Association for Computing Machinery},
address = {New York, NY, USA},
volume = {51},
number = {2},
issn = {0360-0300},
url = {https://doi.org/10.1145/3172867},
doi = {10.1145/3172867},
journal = {ACM Comput. Surv.},
month = feb,
articleno = {35},
numpages = {37}
}

\end{document}